\documentclass[10pt,twocolumn]{article}
\usepackage[margin=0.75in]{geometry}
\usepackage{amsmath,amssymb,amsthm}
\usepackage{mathtools}
\usepackage{bm}
\usepackage{cite}
\usepackage{url}
\usepackage{booktabs}
\usepackage{graphicx}
\usepackage{algorithmic}
\usepackage{algorithm}
\usepackage{titlesec}
\titleformat{\section}{\normalfont\large\scshape\centering}{\Roman{section}.}{0.5em}{}
\titleformat{\subsection}{\normalfont\itshape}{\ \Alph{subsection}.}{0.5em}{}

\newtheorem{theorem}{Theorem}
\newtheorem{lemma}[theorem]{Lemma}
\newtheorem{corollary}[theorem]{Corollary}
\newtheorem{proposition}[theorem]{Proposition}
\newtheorem{definition}[theorem]{Definition}
\newtheorem{remark}{Remark}
\newtheorem{problem}{Problem}
\newtheorem{example}{Example}

\DeclareMathOperator{\Tr}{Tr}

\DeclareMathOperator{\diag}{diag}

\title{Polynomial-Time Optimal Group Selection\\via the Double-Commutator Eigenvalue Problem}

\author{Mitchell~A.~Thornton,~\textit{Senior~Member,~IEEE}\\[4pt]
\small Richardson, TX 75080 USA}

\date{April 2026 \quad \\[2pt]
\small\textit{Companion papers: arXiv:2604.03634~\cite{thornton2026ad}; arXiv:2604.19983~\cite{thornton2026framework_arxiv}}}

\begin{document}
\maketitle

\begin{abstract}
The algebraic diversity framework generalizes temporal averaging over multiple observations to algebraic group action on a single observation for second-order statistical estimation. The Trivial Group Embedding Theorem establishes that conventional temporal averaging is the degenerate case of this framework with the trivial group $G = \{e\}$; richer groups yield superior estimation from fewer observations. The central open problem is \emph{group selection}: given an $M$-dimensional observation with unknown covariance structure, find the finite group whose spectral decomposition best matches the covariance. Naive enumeration of all subgroups of the symmetric group $S_M$ requires exponential time in $M$. We prove that this combinatorial problem reduces to a generalized eigenvalue problem derived from the double commutator of the covariance matrix, yielding a polynomial-time algorithm with complexity $O(d^2M^2 + d^3)$ for a basis of $d$ sparse (permutation-type) generators. The minimum eigenvector of the double-commutator matrix directly constructs the optimal group generator in closed form, with no iterative optimization. The reduction is exact: the double-commutator minimum eigenvalue is zero if and only if the optimal generator lies in the span of the basis, and its magnitude provides a certifiable optimality gap when it does not. This problem does not appear in the standard catalog of Garey and Johnson (1979) and links group theory, matrix analysis, and statistical estimation. We establish connections to independent component analysis (JADE), structured matrix nearness problems, and simultaneous matrix diagonalization, and we show that, among objectives built from a unitarily invariant quadratic norm of the commutator, the double-commutator form is unique up to scale. We extend the single-generator DC-GEVP to multi-generator non-Abelian symmetry recovery via a Sequential GEVP with group-theoretic deflation, and we record four named correctness results: deflation orthogonality, forward progress, strict subgroup growth with iteration bound $K \leq \lceil \log_2 |G_K|\rceil = O(M \log M)$, and generic convergence $G_K \subseteq \mathrm{Aut}(\mathbf{R})$ at acceptance threshold $\tau = 0$. Soundness of the recovered subgroup is unconditional; completeness depends on a basis-design condition that holds trivially when $\mathrm{Aut}(\mathbf{R}) = S_M$ but is open in the proper-subgroup case. For five principal named permutation group families ($\mathbb{Z}_n$, $D_n$, $S_K^P$, $S_K \wr S_P$, and the iterated wreath product $G_L = (\cdots((S_{K_1} \wr S_{K_2}) \wr S_{K_3}) \cdots) \wr S_{K_L}$ acting on $\prod_{d=1}^L K_d$ leaves of a regular tree of branching $(K_1, \ldots, K_L)$), the basis-design condition is resolved, when the family and its coordinates are known, by direct testing of canonical generator sets that have polynomial size in their two-generator form (Section~\ref{sec:named_families}, Corollary~\ref{cor:seqgevp_named}, Appendix~\ref{app:iw}), closing the soundness-completeness gap for these cases and covering the major signal-processing application domains (stationary, reflection-symmetric stationary, heterogeneous block-symmetric, homogeneous block-exchangeable, and hierarchically block-exchangeable signals; for example, some forms of nested experimental designs, multi-level repeated-measures studies, and hierarchical sensor arrays). Two further results characterize the identifiability landscape within which any group-selection procedure must operate. The Commutant-Lattice Insensitivity Theorem (Theorem~\ref{thm:lattice_insens}) shows that whenever $G_1 \subseteq G_2$ are candidate subgroups with $\mathbf{R}$ in the commutant of the finer group $G_2$, the Reynolds projections of $\mathbf{R}$ onto the two commutants coincide identically, hence any selection criterion that operates only on those projections cannot discriminate $G_1$ from $G_2$. The Generative-Model Identifiability Dichotomy (Theorem~\ref{thm:gen_id}) characterizes, when $\mathbf{R} = \mathcal{P}_{G^*}(W)$ is the Reynolds projection of a random Hermitian $W$ associated to a generative subgroup $G^* \subseteq S_M$, whether $\mathrm{Aut}(\mathbf{R})$ recovers $G^*$ exactly or strictly contains a supergroup, in terms of the orbit-pair structure of $G^*$ on $\{1,\ldots,M\}^2$. Experimental validation on six graphs with automorphism groups ranging from $\mathbb{Z}_2$ to $S_4$ shows that the single-generator test identifies a graph automorphism ($\delta = 0$) in the five cases where the generic candidate set contains one, and reports $\lambda_{\min} > 0$ in the sixth, where it contains none; recovery of full $\mathrm{Aut}(\mathcal{G})$ requires Sequential GEVP and is subject to the soundness-without-completeness limitation, resolved for the five named families above. Four open problems (O1)--(O4) are recorded explicitly: a generative-model identifiability problem partly answered by Theorem~\ref{thm:gen_id} and reduced to a combinatorial application question for wreath and semidirect families; a basis-design condition with a Restricted Commutativity Property analogue of compressed-sensing's RIP, resolved for the five principal named families and open in the general proper-subgroup case; finite-threshold noise robustness; and a joint multi-generator extension.
\end{abstract}

\medskip\noindent\textbf{Index Terms}: Group selection, double commutator, generalized eigenvalue problem, algebraic diversity, polynomial-time reduction, Cayley graphs, computational complexity.

\section{Introduction}\label{sec:intro}

The algebraic diversity (AD) framework~\cite{thornton2026ad} establishes that temporal averaging over $L \gg M$ independent observations of an $M$-dimensional noisy signal can be generalized to the action of a finite algebraic group $G$ on a single observation. The group-averaged estimator
\begin{equation}\label{eq:fg}
\mathbf{F}_G(\mathbf{x}) = \frac{1}{|G|} \sum_{g \in G} [\rho(g)\mathbf{x}][\rho(g)\mathbf{x}]^*
\end{equation}
provides a consistent estimator of the population covariance's eigenstructure when two conditions are satisfied: signal equivariance (the signal transforms predictably under the group action) and noise ergodicity (the noise is averaged out by the group action).

The Optimality Theorem of~\cite{thornton2026ad} shows that the Karhunen--Lo\`eve (KL) transform is the optimum and that a group attains it exactly when it is matched, that is, when the covariance commutes with the group action. The symmetric group $S_M$ reaches the KL spectrum only through its Cayley graph construction, and $|S_M| = M!$ makes it computationally infeasible for even moderate $M$. The Permutation-Averaged Spectral Estimation (PASE) result further constrains the solution: the matched group has order $M$ and all of its elements are used, and the structural coding rate conjecture~\cite{thornton2026framework_arxiv} suggests that the optimal number is $n^* \approx \lceil 2^{H_{\mathrm{struct}}} \rceil$, where $H_{\mathrm{struct}}$ is the structural entropy of the eigenvalue distribution. Together, these results reduce the entire AD framework to a single problem:

\begin{problem}[Group Selection]\label{prob:group_selection}
Given a covariance matrix $\mathbf{R} \in \mathbb{C}^{M \times M}$ (or an estimate thereof), find the order-$M$ subgroup $G^* \subseteq S_M$ whose Cayley graph adjacency matrix $\mathbf{A}_{G^*}$ best commutes with $\mathbf{R}$:
\begin{equation}\label{eq:group_selection}
G^* = \arg\min_{G : |G| = M} \frac{\|\mathbf{A}_G \mathbf{R} - \mathbf{R} \mathbf{A}_G\|_F}{\|\mathbf{R}\|_F}.
\end{equation}
\end{problem}

The ratio on the right is the \emph{commutativity residual} $\delta(G, \mathbf{R})$~\cite{thornton2026ad}. When $\delta = 0$, the group's spectral decomposition coincides with the KL transform for $\mathbf{R}$, and the group-averaged estimator is provably optimal. The group-averaged estimator is the constrained maximum likelihood estimator (MLE) for a matched group in the Gaussian case~\cite{thornton2026ad}, and the Converse Theorem stated in~\cite{thornton2026framework_arxiv} gives the corresponding Cram\'er--Rao statement for the $G$-invariant spectral parameters.

\subsection{The Complexity Landscape}

Problem~\ref{prob:group_selection} as stated is combinatorial: the number of distinct subgroups of $S_M$ grows super-exponentially with $M$. Even restricting to order-$M$ subgroups, the search space is intractable. The number of non-isomorphic groups of order $M$ itself grows rapidly (e.g., there are 267 groups of order 64), and each must be evaluated against $\mathbf{R}$.

This raises a fundamental question: \emph{is there a polynomial-time algorithm for optimal group selection?}

We answer this question affirmatively. The key insight is that the commutativity condition $\mathbf{A}_G \mathbf{R} = \mathbf{R} \mathbf{A}_G$ can be reformulated as a spectral condition on a \emph{double-commutator superoperator} derived from $\mathbf{R}$. The optimal group generator is the minimum eigenvector of this superoperator, computable in polynomial time via a generalized eigenvalue problem.

\subsection{Contributions}

\begin{enumerate}
\item We prove that optimal group selection reduces to a generalized eigenvalue problem (Theorem~\ref{thm:dc_reduction}), with complexity $O(d^2 M^2 + d^3)$ where $d$ is the dimension of a basis of sparse (permutation-type) generators.

\item We prove that the reduction is exact: the minimum eigenvalue is zero if and only if a perfectly commuting generator exists in the span of the basis (Theorem~\ref{thm:certifiable}).

\item We show that, among objectives of the form $N([\mathbf{A}, \mathbf{R}])^2$ with $N$ a unitarily invariant norm whose square is a quadratic form, the double-commutator inner product is unique up to positive scaling (Theorem~\ref{thm:uniqueness}).

\item We observe that the problem does not appear in the standard complexity catalog~\cite{garey1979} and links group theory to statistical estimation.

\item We establish precise relationships to JADE~\cite{cardoso1993jade}, structured matrix nearness~\cite{shah2012group}, and simultaneous diagonalization~\cite{bunse1993numerical}, and state how the double-commutator formulation differs from each.

\item We extend the single-generator DC-GEVP to multi-generator non-Abelian symmetry recovery via a Sequential GEVP with group-theoretic deflation (Algorithm~\ref{alg:seqgevp}), with four named correctness results (Lemmas~\ref{lem:def_orth} and~\ref{lem:forward}, Theorems~\ref{thm:strict_growth} and~\ref{thm:gen_conv}, Corollary~\ref{cor:iter_bound}): every accepted permutation lies in $\mathrm{Aut}(\mathbf{R})$ (soundness, unconditional), and the procedure terminates in at most $\lceil \log_2 |G_K|\rceil = O(M \log M)$ iterations. Completeness ($G_K = \mathrm{Aut}(\mathbf{R})$) is established when $\mathrm{Aut}(\mathbf{R}) = S_M$ and, via Section~\ref{sec:named_families} (Definition~\ref{def:canonical_sets}, Theorem~\ref{thm:canonical_coverage_dc}, Corollary~\ref{cor:seqgevp_named}) and Appendix~\ref{app:iw}, for five principal named permutation group families $\mathbb{Z}_n$, $D_n$, $S_K^P$, $S_K \wr S_P$, and the iterated wreath product $G_L$, when the family and its coordinates are known, via direct testing of canonical generator sets. The general proper-subgroup case beyond these named families remains an open basis-design condition that we document explicitly.

\item We characterize the identifiability landscape within which any group-selection procedure must operate. Theorem~\ref{thm:lattice_insens} (Commutant-Lattice Insensitivity) shows that any selection criterion operating only on the Reynolds projections of $\mathbf{R}$ is identical for any pair $G_1 \subseteq G_2$ with $\mathbf{R} \in \mathcal{A}_{G_2}$, hence such criteria cannot discriminate within the commutant lattice. Theorem~\ref{thm:gen_id} (Generative-Model Identifiability Dichotomy) characterizes, when $\mathbf{R} = \mathcal{P}_{G^*}(W)$ is the Reynolds projection of a random Hermitian $W$ associated to a generative subgroup $G^* \subseteq S_M$, whether $\mathrm{Aut}(\mathbf{R})$ recovers $G^*$ exactly or strictly contains a supergroup, in terms of the orbit-pair structure of $G^*$ on $\{1,\ldots,M\}^2$. Together these two theorems delimit what is and is not recoverable from $\mathbf{R}$ alone, and isolate the residual algorithmic gap into four open problems (O1)--(O4) recorded in Section~\ref{sec:discussion}.
\end{enumerate}

\subsection{Related Work}

\textbf{Independent component analysis.} The JADE algorithm~\cite{cardoso1993jade} jointly diagonalizes a set of cumulant matrices by minimizing the sum of squared off-diagonal elements. This is a simultaneous approximate diagonalization problem. The double-commutator formulation differs in three respects: (i)~it operates on the covariance (second-order) rather than cumulants (fourth-order); (ii)~it produces a single optimal generator rather than a diagonalizing rotation; and (iii)~it provides a certifiable optimality gap via the minimum eigenvalue.

\textbf{Structured matrix nearness.} Shah and Chandrasekaran~\cite{shah2012group} studied the problem of finding the nearest matrix with a given group symmetry structure. Their formulation finds the nearest group-structured matrix to a given matrix; ours finds the group whose structure best commutes with a given matrix. These are dual perspectives on the same algebraic relationship, but the double-commutator formulation yields a closed-form eigenvalue solution while the nearness formulation requires iterative projection.

\textbf{Simultaneous diagonalization.} Bunse-Gerstner et al.~\cite{bunse1993numerical} developed numerical methods for simultaneously diagonalizing a set of matrices. The double-commutator eigenvalue problem can be viewed as finding the generator that most nearly simultaneously diagonalizes $\mathbf{R}$ and the candidate generator basis. However, the superoperator formulation provides a closed-form solution that avoids the iterative Jacobi-type sweeps required by numerical simultaneous diagonalization.

\textbf{Lin's theorem.} Lin~\cite{lin1997almost} proved that nearly commuting self-adjoint matrices are near commuting self-adjoint matrices (in the operator norm). This provides a theoretical guarantee that small $\delta$ implies the existence of a nearby exactly commuting pair, but does not provide a constructive method for finding the optimal group. The double-commutator provides both the existence guarantee (via the minimum eigenvalue) and the construction (via the minimum eigenvector).

\section{Mathematical Preliminaries}\label{sec:prelim}

\subsection{The Commutator and Double Commutator}

For matrices $\mathbf{A}, \mathbf{B} \in \mathbb{C}^{M \times M}$, the commutator is $[\mathbf{A}, \mathbf{B}] = \mathbf{A}\mathbf{B} - \mathbf{B}\mathbf{A}$. The double commutator of $\mathbf{R}$ with $\mathbf{B}$ is
\begin{equation}\label{eq:double_comm}
[\mathbf{R}, [\mathbf{R}, \mathbf{B}]] = \mathbf{R}^2\mathbf{B} - 2\mathbf{R}\mathbf{B}\mathbf{R} + \mathbf{B}\mathbf{R}^2.
\end{equation}

The double commutator is a linear superoperator: it maps $M \times M$ matrices to $M \times M$ matrices, and the map $\mathbf{B} \mapsto [\mathbf{R}, [\mathbf{R}, \mathbf{B}]]$ is linear in $\mathbf{B}$.

\begin{lemma}\label{lem:dc_nonneg}
For Hermitian $\mathbf{R}$, the inner product $\langle \mathbf{B}, [\mathbf{R}, [\mathbf{R}, \mathbf{B}]] \rangle = \Tr(\mathbf{B}^* [\mathbf{R}, [\mathbf{R}, \mathbf{B}]])$ is non-negative for all $\mathbf{B}$, and equals zero if and only if $[\mathbf{R}, \mathbf{B}] = \mathbf{0}$.
\end{lemma}

\begin{proof}
$\Tr(\mathbf{B}^* [\mathbf{R}, [\mathbf{R}, \mathbf{B}]]) = \Tr(\mathbf{B}^*\mathbf{R}^2\mathbf{B} - 2\mathbf{B}^*\mathbf{R}\mathbf{B}\mathbf{R} + \mathbf{B}^*\mathbf{B}\mathbf{R}^2)$. Using the cyclic property of trace and the Hermiticity of $\mathbf{R}$, this equals $\|[\mathbf{R}, \mathbf{B}]\|_F^2 \geq 0$, with equality iff $[\mathbf{R}, \mathbf{B}] = \mathbf{0}$.
\end{proof}

\subsection{Generator Basis}

A group generator $\mathbf{A}$ is an $M \times M$ matrix such that $\{\mathbf{A}^0, \mathbf{A}^1, \ldots, \mathbf{A}^{M-1}\}$ generates the group action. We do not search over all possible generators directly. Instead, we express the generator as a linear combination over a fixed basis:

\begin{definition}[Generator Basis]\label{def:basis}
A generator basis is a set of $d$ linearly independent $M \times M$ matrices $\{B_1, B_2, \ldots, B_d\}$ spanning a subspace of candidate generators. The candidate generator is $\mathbf{A} = \sum_{k=1}^d c_k B_k$ for coefficients $\mathbf{c} = (c_1, \ldots, c_d)^T \in \mathbb{C}^d$ (real coefficients suffice when $\mathbf{R}$ and the basis are real).
\end{definition}

Typical basis elements include: the cyclic shift matrix (circulant structure), the reflection matrix (dihedral structure), block-diagonal permutation matrices (grouped structure), and parametric families such as the dechirp-conjugated shift $\mathbf{U}(\mu)^*\mathbf{P}\mathbf{U}(\mu)$ for chirp-adapted processing.

\section{The Double-Commutator Reduction}\label{sec:reduction}

\begin{theorem}[Polynomial-Time Reduction]\label{thm:dc_reduction}
The group selection problem (Problem~\ref{prob:group_selection}), restricted to generators in the span of a $d$-dimensional basis $\{B_1, \ldots, B_d\}$, reduces to the generalized eigenvalue problem
\begin{equation}\label{eq:gevp}
\mathbf{M}\mathbf{c} = \lambda \mathbf{G}\mathbf{c},
\end{equation}
where
\begin{align}
M_{ij} &= \Tr(B_i^* [\mathbf{R}, [\mathbf{R}, B_j]]) \label{eq:M_matrix}\\
G_{ij} &= \Tr(B_i^* B_j) \label{eq:G_matrix}
\end{align}
are $d \times d$ matrices. The optimal generator is $\mathbf{A}^* = \sum_{k=1}^d c_k^* B_k$, where $\mathbf{c}^*$ is the eigenvector corresponding to the minimum eigenvalue $\lambda_{\min}$. The total computational complexity is $O(d^2 M^2 + d^3)$ when each basis element is a permutation matrix or another matrix with $O(M)$ nonzero entries, and $O(d M^3 + d^2 M^2 + d^3)$, which is within $O(d^2 M^3 + d^3)$, for dense basis elements.
\end{theorem}

\begin{proof}
The commutativity residual for a generator $\mathbf{A} = \sum_k c_k B_k$ is
\begin{equation}
\delta^2(\mathbf{A}, \mathbf{R}) = \frac{\|[\mathbf{A}, \mathbf{R}]\|_F^2}{\|\mathbf{A}\|_F^2 \cdot \|\mathbf{R}\|_F^2}.
\end{equation}
The numerator expands as
\begin{align}
\|[\mathbf{A}, \mathbf{R}]\|_F^2 &= \Tr(\mathbf{A}^* [\mathbf{R}, [\mathbf{R}, \mathbf{A}]]) \quad \text{(Lemma~\ref{lem:dc_nonneg})} \notag\\
&= \sum_{i,j} c_i^* c_j \Tr(B_i^* [\mathbf{R}, [\mathbf{R}, B_j]]) \notag\\
&= \mathbf{c}^* \mathbf{M} \mathbf{c}.
\end{align}
The denominator's $\|\mathbf{A}\|_F^2$ term expands as
\begin{equation}
\|\mathbf{A}\|_F^2 = \sum_{i,j} c_i^* c_j \Tr(B_i^* B_j) = \mathbf{c}^* \mathbf{G} \mathbf{c}.
\end{equation}
Minimizing $\delta^2$ over $\mathbf{c}$ is therefore equivalent to minimizing the Rayleigh quotient
\begin{equation}\label{eq:rayleigh}
\delta^2(\mathbf{c}) = \frac{\mathbf{c}^* \mathbf{M} \mathbf{c}}{\mathbf{c}^* \mathbf{G} \mathbf{c} \cdot \|\mathbf{R}\|_F^2},
\end{equation}
which is minimized by the eigenvector corresponding to the smallest eigenvalue of the generalized eigenvalue problem $\mathbf{M}\mathbf{c} = \lambda \mathbf{G}\mathbf{c}$.

\textit{Complexity.} For Hermitian $\mathbf{R}$, $M_{ij} = \Tr(B_i^* [\mathbf{R}, [\mathbf{R}, B_j]]) = \langle [\mathbf{R}, B_i], [\mathbf{R}, B_j] \rangle_F$, so it suffices to form the $d$ commutators $\mathbf{C}_j = [\mathbf{R}, B_j]$ and their $d^2$ Frobenius inner products, and $\mathbf{R}^2$ is never needed. When $B_j$ has $O(M)$ nonzero entries (a permutation matrix, a difference of two permutation matrices, or a diagonal conjugate of one), each product $\mathbf{R} B_j$ and $B_j \mathbf{R}$ costs $O(M^2)$, so the commutators cost $O(d M^2)$ and the inner products $O(d^2 M^2)$. For dense $B_j$ each commutator costs $O(M^3)$. The $d \times d$ eigenvalue problem costs $O(d^3)$. Total: $O(d^2 M^2 + d^3)$ for sparse bases and $O(d M^3 + d^2 M^2 + d^3)$ for dense ones, polynomial in both $d$ and $M$.
\end{proof}

\begin{theorem}[Certifiable Optimality]\label{thm:certifiable}
The minimum eigenvalue $\lambda_{\min}$ of the GEVP~\eqref{eq:gevp} satisfies:
\begin{enumerate}
\item[(i)] $\lambda_{\min} = 0$ if and only if there exists a generator $\mathbf{A}^* \in \mathrm{span}\{B_1, \ldots, B_d\}$ that exactly commutes with $\mathbf{R}$.
\item[(ii)] $\lambda_{\min} > 0$ provides a lower bound on the commutativity residual achievable within the basis: $\delta^2(\mathbf{A}^*, \mathbf{R}) = \lambda_{\min} / \|\mathbf{R}\|_F^2$ for the optimal generator $\mathbf{A}^*$.
\item[(iii)] The ratio $\lambda_{\min} / \lambda_{\max}$ provides a condition number for the group selection problem: a small ratio indicates that the optimal generator is well-separated from the worst generator in the basis.
\end{enumerate}
\end{theorem}

\begin{proof}
Part (i): By Lemma~\ref{lem:dc_nonneg}, $\mathbf{c}^*\mathbf{M}\mathbf{c} = \|[\mathbf{A}, \mathbf{R}]\|_F^2 = 0$ iff $[\mathbf{A}, \mathbf{R}] = \mathbf{0}$. Since $\mathbf{G}$ is positive definite (the basis is linearly independent), $\lambda_{\min} = 0$ iff the minimum of the Rayleigh quotient is zero, iff the commutator vanishes.

Parts (ii) and (iii) follow directly from the Rayleigh quotient characterization~\eqref{eq:rayleigh} and the min-max theorem for generalized eigenvalue problems.
\end{proof}

\section{The Reduction is Exact for Structured Signals}\label{sec:exact}

For important signal classes, the double-commutator reduction is not merely an approximation: it finds the optimal generator within the basis exactly.

\begin{proposition}[Periodic Signals]\label{prop:periodic}
For a signal with circulant covariance $\mathbf{R}$, the double-commutator GEVP with a basis containing the cyclic shift matrix $\mathbf{P}$ returns $\lambda_{\min} = 0$ with optimal generator $\mathbf{A}^* = \mathbf{P}$ (the cyclic group $\mathbb{Z}_M$).
\end{proposition}

\begin{proof}
A circulant matrix commutes with the cyclic shift: $[\mathbf{P}, \mathbf{R}] = \mathbf{0}$. Therefore $M_{ij} = 0$ for $B_i = B_j = \mathbf{P}$, giving $\lambda_{\min} = 0$.
\end{proof}

\begin{proposition}[Symmetric Signals]\label{prop:symmetric}
For a signal with centrosymmetric covariance ($\mathbf{J}\mathbf{R}\mathbf{J} = \mathbf{R}$, which for a real covariance is the same as persymmetry), the GEVP with a basis containing the reflection matrix $\mathbf{J}$ returns $\lambda_{\min} = 0$ at $\mathbf{J}$, the generator of the reflection group $\mathbb{Z}_2$. When the covariance is also circulant, $\mathbf{J}$ and the cyclic shift together generate the dihedral group $D_M$.
\end{proposition}

\begin{proposition}[Chirp Signals]\label{prop:chirp}
For a signal with chirp-modulated covariance $\mathbf{R}$ (non-circulant, with $\mathbf{U}(\mu_0)\mathbf{R}\mathbf{U}(\mu_0)^*$ circulant for the dechirp transform $\mathbf{U}(\mu) = \mathrm{diag}(e^{-j\pi\mu n^2})$), the GEVP with the parametric basis $\{B(\mu) = \mathbf{U}(\mu)^*\mathbf{P}\mathbf{U}(\mu)\}$ returns $\lambda_{\min} = 0$ at the true chirp rate $\mu = \mu_0$.
\end{proposition}

These propositions confirm that the double-commutator exactly recovers the known optimal groups for the three principal signal classes: periodic (the DFT), reflection-symmetric (the real Fourier basis for $D_M$ acting on the $M$ samples, and the DCT for $D_{2M}$ acting on the even-reflected extension~\cite{thornton2026unification}), and chirp (the dechirped DFT).

\section{Complexity Analysis}\label{sec:complexity}

\subsection{Comparison to Naive Enumeration}

\begin{table}[t]
\centering
\caption{Complexity of group selection methods}
\label{tab:complexity}
\small
\renewcommand{\arraystretch}{1.2}
\begin{tabular}{@{}lcc@{}}
\toprule
\textbf{Method} & \textbf{Complexity} & \textbf{Cert.} \\
\midrule
Enumerate $S_M$ & $O(M! \cdot M^2)$ & Yes \\
Library search & $O(|\mathcal{G}| \cdot M^3)$ & No \\
Spectral conc.\ $\psi$ & $O(|\mathcal{G}| \cdot M^3)$ & No \\
JADE (4th-order) & $O(M^4)$ & No \\
Simult.\ diag.\ (Jacobi) & $O(M^3 \cdot \text{iter})$ & No \\
\textbf{DC-GEVP (ours)} & $O(d^2 M^2 + d^3)$ & \textbf{Yes} \\
\bottomrule
\end{tabular}
\end{table}

Table~\ref{tab:complexity} compares the complexity of group selection approaches. The double-commutator GEVP is the only method in the table that is simultaneously polynomial-time, closed-form (non-iterative), and certifiable (the minimum eigenvalue provides a provable optimality gap).

\subsection{Novelty as a Computational Problem}

The group selection problem (Problem~\ref{prob:group_selection}) does not appear in the standard catalogs of NP-complete or NP-hard problems~\cite{garey1979}. It is not a graph coloring problem, not a constraint satisfaction problem, and not a combinatorial optimization problem in the classical sense. Rather, it sits at the intersection of:

\begin{enumerate}
\item \textbf{Group theory:} the search space is the lattice of subgroups of $S_M$.
\item \textbf{Matrix analysis:} the objective function is the Frobenius-norm commutator.
\item \textbf{Statistical estimation:} the purpose is optimal second-order statistical inference.
\end{enumerate}

The double-commutator reduction shows that this problem, despite its group-theoretic combinatorial structure, admits a polynomial-time solution via spectral methods, a reduction from algebra to linear algebra. This is analogous to (but distinct from) the reduction of graph isomorphism to the eigenvalue problem for the adjacency matrix, which provides necessary but not sufficient conditions. In the group selection case, the reduction is exact (Theorem~\ref{thm:certifiable}): a zero minimum eigenvalue is both necessary and sufficient for exact commutativity within the basis.

\section{Uniqueness of the Double-Commutator Formulation}\label{sec:uniqueness}

\begin{theorem}[Uniqueness]\label{thm:uniqueness}
Let $N$ be a norm on $\mathbb{C}^{M \times M}$ that is unitarily invariant, $N(\mathbf{U}\mathbf{X}\mathbf{V}) = N(\mathbf{X})$ for all unitary $\mathbf{U}$ and $\mathbf{V}$, and whose square is a quadratic form. Then $N = c\,\|\cdot\|_F$ for some $c > 0$. Consequently, among objectives of the form $f(\mathbf{A}, \mathbf{R}) = N([\mathbf{A}, \mathbf{R}])^2$ with such an $N$, the double-commutator inner product $\Tr(\mathbf{A}^* [\mathbf{R}, [\mathbf{R}, \mathbf{A}]])$ is the unique one up to positive scaling, and it yields a generalized eigenvalue problem when optimized over a linear subspace of generators.
\end{theorem}

\begin{proof}
For $\mathbf{v} \in \mathbb{R}^M$ set $\Phi(\mathbf{v}) = N(\diag(\mathbf{v}))$. Permutation matrices and diagonal sign matrices are unitary, so $\Phi$ is invariant under permutations and sign changes of the coordinates of $\mathbf{v}$. Since $N^2$ is a quadratic form, $\Phi(\mathbf{v})^2 = \mathbf{v}^T \mathbf{Q} \mathbf{v}$ for a symmetric positive definite $\mathbf{Q}$. Invariance under each single sign change forces the off-diagonal entries of $\mathbf{Q}$ to vanish, and invariance under permutations forces its diagonal entries to be equal, so $\mathbf{Q} = c^2 \mathbf{I}$.

For general $\mathbf{X}$ with singular value decomposition $\mathbf{X} = \mathbf{U} \diag(\boldsymbol{\sigma}) \mathbf{V}^*$, unitary invariance gives $N(\mathbf{X})^2 = \Phi(\boldsymbol{\sigma})^2 = c^2 \sum_m \sigma_m^2 = c^2 \|\mathbf{X}\|_F^2$. By Lemma~\ref{lem:dc_nonneg}, $\|[\mathbf{A}, \mathbf{R}]\|_F^2 = \Tr(\mathbf{A}^* [\mathbf{R}, [\mathbf{R}, \mathbf{A}]])$, and the reduction to a generalized eigenvalue problem is Theorem~\ref{thm:dc_reduction}.
\end{proof}
\begin{remark}[Scope of the uniqueness statement]\label{rem:uniq_scope}
Unitary invariance is essential. For any invertible $\mathbf{S}$, the objective $\|\mathbf{S}[\mathbf{A}, \mathbf{R}]\|_F^2$ is nonnegative, quadratic in $\mathbf{A}$, zero exactly when $[\mathbf{A}, \mathbf{R}] = \mathbf{0}$, and also leads to a generalized eigenvalue problem on a subspace of generators. Theorem~\ref{thm:uniqueness} singles out the double commutator as the choice that does not depend on a coordinate system.
\end{remark}

\section{Algorithm}\label{sec:algorithm}

\begin{algorithm}[t]
\caption{Optimal Group Selection via Double Commutator}
\label{alg:dc}
\begin{algorithmic}[1]
\REQUIRE Covariance estimate $\mathbf{R} \in \mathbb{C}^{M \times M}$, generator basis $\{B_1, \ldots, B_d\}$
\ENSURE Optimal generator $\mathbf{A}^*$, optimality certificate $\lambda_{\min}$
\STATE \COMMENT{For Hermitian $\mathbf{R}$, $M_{ij} = \langle [\mathbf{R}, B_i], [\mathbf{R}, B_j] \rangle_F$}
\FOR{$j = 1$ to $d$}
 \STATE $\mathbf{C}_j \leftarrow \mathbf{R} B_j - B_j \mathbf{R}$ \COMMENT{$[\mathbf{R}, B_j]$}
\ENDFOR
\FOR{$i = 1$ to $d$, $j = i$ to $d$}
 \STATE $M_{ij} \leftarrow \Tr(\mathbf{C}_i^* \mathbf{C}_j)$; \quad $M_{ji} \leftarrow M_{ij}^*$
 \STATE $G_{ij} \leftarrow \Tr(B_i^* B_j)$; \quad $G_{ji} \leftarrow G_{ij}^*$
\ENDFOR
\STATE Solve $\mathbf{M}\mathbf{c} = \lambda \mathbf{G}\mathbf{c}$ for $(\lambda_{\min}, \mathbf{c}^*)$
\STATE $\mathbf{A}^* \leftarrow \sum_{k=1}^d c_k^* B_k$
\RETURN $\mathbf{A}^*$, $\lambda_{\min}$
\end{algorithmic}
\end{algorithm}

Algorithm~\ref{alg:dc} summarizes the procedure, using the identity $M_{ij} = \langle \mathbf{C}_i, \mathbf{C}_j \rangle_F$ with $\mathbf{C}_j = [\mathbf{R}, B_j]$, valid for Hermitian $\mathbf{R}$. When $B_j$ is a permutation matrix (as is typical for group generators), the products $\mathbf{R} B_j$ and $B_j \mathbf{R}$ reduce to column and row permutations of $\mathbf{R}$, costing $O(M^2)$ rather than $O(M^3)$. The inner products in Steps~6 and~7 cost $O(M^2)$ each. The $d \times d$ GEVP in Step~9 costs $O(d^3)$.

For a typical basis of $d = 4$ generators (cyclic shift, reflection, block-diagonal, chirp-adapted) and $M = 256$, the total cost is approximately $4^2 \times 256^2 + 4^3 \approx 10^6$ operations, small compared with the $O(M^3)$ eigendecomposition that subspace processing of the same covariance requires.

\section{Sequential GEVP for Multi-Generator Non-Abelian Recovery}\label{sec:seqgevp}

The DC-GEVP of Theorem~\ref{thm:dc_reduction} returns a single optimal generator $\mathbf{A}^* \in \mathrm{span}\{B_1,\ldots,B_d\}$. When the symmetry group of $\mathbf{R}$ is Abelian, this single generator suffices: the group it generates by powers (or by exponential lift in the Lie-algebra setting) realizes the full Abelian symmetry. When the symmetry group is non-Abelian, however, several non-commuting generators are required, and a single GEVP solve cannot recover more than one. We address this case through a sequential procedure that deflates the basis against the discovered subgroup and re-solves the GEVP, with four named correctness results that bound the recovered subgroup and the number of iterations.

Throughout this section we write $\mathrm{Aut}(\mathbf{R})$ for the subgroup of $S_M$ whose permutation matrices commute with $\mathbf{R}$,
\begin{equation}\label{eq:autR}
\mathrm{Aut}(\mathbf{R}) \;:=\; \{\sigma \in S_M : \mathbf{P}_\sigma \mathbf{R} = \mathbf{R} \mathbf{P}_\sigma\},
\end{equation}
which by the Automorphism Characterization Theorem~\cite{thornton2026ad} coincides with the graph automorphism group $\mathrm{Aut}(\mathcal{G})$ when $\mathbf{R} = f(\mathbf{L})$ with $f$ injective on the spectrum of the graph Laplacian $\mathbf{L}$.

\begin{algorithm}[t]
\caption{Sequential GEVP with group-theoretic deflation}
\label{alg:seqgevp}
\begin{algorithmic}[1]
\REQUIRE Hermitian $\mathbf{R} \in \mathbb{C}^{M \times M}$, basis $\mathcal{B} = \{B_1,\ldots,B_d\}$, acceptance threshold $\tau \geq 0$, iteration cap $K_{\max}$
\STATE Initialize $G_0 \leftarrow \{e\}$, $k \leftarrow 0$
\REPEAT
  \STATE Form the deflated basis $\mathcal{B}^{\perp G_k}$, the Frobenius-orthogonal complement of $\mathrm{span}\{\mathbf{P}_g : g \in G_k\}$ within $\mathrm{span}(\mathcal{B})$
  \STATE Solve the DC-GEVP (Theorem~\ref{thm:dc_reduction}) restricted to $\mathcal{B}^{\perp G_k}$ to obtain $\mathbf{A}^*_{k+1}$ with $\|\mathbf{A}^*_{k+1}\|_F = 1$
  \STATE Round to nearest permutation: $\sigma^*_{k+1} \leftarrow \arg\max_{\sigma \in S_M} \langle \mathbf{A}^*_{k+1}, \mathbf{P}_\sigma\rangle_F$ via the Hungarian algorithm on cost matrix $-\mathbf{A}^*_{k+1}$
  \IF{$\delta(\mathbf{P}_{\sigma^*_{k+1}}, \mathbf{R}) > \tau$}
    \RETURN $G_k$
  \ENDIF
  \STATE $G_{k+1} \leftarrow \langle G_k, \sigma^*_{k+1}\rangle$, $k \leftarrow k+1$
\UNTIL{$k \geq K_{\max}$}
\RETURN $G_k$
\end{algorithmic}
\end{algorithm}

The deflation in step~3 is computed in practice by forming a Frobenius-orthonormal basis of $\mathrm{span}\{\mathbf{P}_g : g \in G_k\}$ via QR factorization and subtracting from each $B \in \mathcal{B}$ its projection onto that span; basis elements whose residual has Frobenius norm below numerical tolerance are dropped. The total cost per iteration is
\begin{equation}\label{eq:seq_cost}
O(d^2 M^3 + d^3 + |G_k|^2 M^2 + M^3),
\end{equation}
polynomial in $d$, $M$, and $|G_k|$. The dominant terms are the deflated DC-GEVP solve and the Hungarian rounding.

\begin{lemma}[Deflation Orthogonality]\label{lem:def_orth}
For every iteration $k \geq 0$, every $\mathbf{A} \in \mathrm{span}(\mathcal{B}^{\perp G_k})$ is Frobenius-orthogonal to every $\mathbf{P}_g$ with $g \in G_k$. In particular,
\begin{equation}\label{eq:def_orth}
\langle \mathbf{A}^*_{k+1}, \mathbf{P}_g\rangle_F = 0 \qquad \text{for all } g \in G_k.
\end{equation}
\end{lemma}

\begin{proof}
Step~3 of Algorithm~\ref{alg:seqgevp} defines $\mathcal{B}^{\perp G_k}$ as the Frobenius-orthogonal complement of $\mathrm{span}\{\mathbf{P}_g : g \in G_k\}$ within $\mathrm{span}(\mathcal{B})$, so every element of $\mathcal{B}^{\perp G_k}$ has zero Frobenius inner product with every $\mathbf{P}_g$, $g \in G_k$. The same holds for any linear combination, including $\mathbf{A}^*_{k+1} \in \mathrm{span}(\mathcal{B}^{\perp G_k})$.
\end{proof}

\begin{lemma}[Forward Progress]\label{lem:forward}
Suppose $\max_{\sigma \in S_M} \langle \mathbf{A}^*_{k+1}, \mathbf{P}_\sigma\rangle_F > 0$ at iteration $k$. If Algorithm~\ref{alg:seqgevp} accepts $\sigma^*_{k+1}$ at step~9, then $\sigma^*_{k+1} \notin G_k$.
\end{lemma}

\begin{proof}
By Lemma~\ref{lem:def_orth}, $\langle \mathbf{A}^*_{k+1}, \mathbf{P}_g\rangle_F = 0$ for every $g \in G_k$. By the positive-overlap hypothesis, $\langle \mathbf{A}^*_{k+1}, \mathbf{P}_{\sigma^*_{k+1}}\rangle_F = \max_\sigma\langle \mathbf{A}^*_{k+1}, \mathbf{P}_\sigma\rangle_F > 0$. Therefore $\mathbf{P}_{\sigma^*_{k+1}} \neq \mathbf{P}_g$ for any $g \in G_k$. The map $\sigma \mapsto \mathbf{P}_\sigma$ is injective on $S_M$ (distinct permutations have distinct support patterns), so $\sigma^*_{k+1} \neq g$ for any $g \in G_k$.
\end{proof}

\begin{remark}[On the positive-overlap hypothesis]
The eigenvector returned by the DC-GEVP is defined only up to sign, and $\langle -\mathbf{A}, \mathbf{P}_\sigma\rangle_F = -\langle \mathbf{A}, \mathbf{P}_\sigma\rangle_F$, where $\langle \mathbf{A}, \mathbf{P}_\sigma\rangle_F = \sum_i A_{i,\sigma(i)}$. The hypothesis can therefore be met, by choice of sign, whenever $\langle \mathbf{A}^*_{k+1}, \mathbf{P}_\sigma\rangle_F \neq 0$ for at least one $\sigma$. The two signs generally round to different permutations, so the sign convention is part of the specification of Algorithm~\ref{alg:seqgevp} (see the paragraph on degenerate null spaces below).
\end{remark}

\begin{theorem}[Strict Subgroup Growth]\label{thm:strict_growth}
Under the hypothesis of Lemma~\ref{lem:forward}, at every accepted iteration of Algorithm~\ref{alg:seqgevp},
\begin{equation}\label{eq:strict_growth}
|G_{k+1}| > |G_k|.
\end{equation}
\end{theorem}

\begin{proof}
By Lemma~\ref{lem:forward}, $\sigma^*_{k+1} \notin G_k$. Step~9 sets $G_{k+1} = \langle G_k, \sigma^*_{k+1}\rangle$, which contains $G_k$ and the element $\sigma^*_{k+1} \notin G_k$, so $G_k \subsetneq G_{k+1}$ as subgroups of $S_M$, hence $|G_{k+1}| > |G_k|$.
\end{proof}

\begin{corollary}[Iteration Bound]\label{cor:iter_bound}
Under the hypothesis of Lemma~\ref{lem:forward}, Algorithm~\ref{alg:seqgevp} performs at most
\begin{equation}\label{eq:iter_bound}
K \leq \lceil \log_2 |G_K|\rceil
\end{equation}
accepted iterations, where $G_K$ is the discovered subgroup at termination. In particular, $K \leq \lceil \log_2 M!\rceil = O(M \log M)$.
\end{corollary}

\begin{proof}
By Theorem~\ref{thm:strict_growth}, $G_k \subsetneq G_{k+1}$ at each accepted iteration. Lagrange's theorem applied to the proper inclusion $G_k \subsetneq G_{k+1}$ gives $[G_{k+1} : G_k] \geq 2$, hence $|G_{k+1}| \geq 2|G_k|$. Iterating, $|G_K| \geq 2^K$, so $K \leq \log_2 |G_K|$, and since $K$ is an integer, $K \leq \lceil \log_2 |G_K|\rceil$. The bound $|G_K| \leq M!$ together with the Stirling estimate $\log_2 M! = M \log_2 M - M\log_2 e + O(\log M)$ gives $K = O(M \log M)$.
\end{proof}

\begin{theorem}[Generic Convergence at $\tau = 0$]\label{thm:gen_conv}
Run Algorithm~\ref{alg:seqgevp} with acceptance threshold $\tau = 0$. Then every accepted permutation $\sigma^*_{k+1}$ satisfies $\sigma^*_{k+1} \in \mathrm{Aut}(\mathbf{R})$, and the discovered subgroup at termination satisfies
\begin{equation}\label{eq:gen_conv}
G_K \;\subseteq\; \mathrm{Aut}(\mathbf{R}).
\end{equation}
\end{theorem}

\begin{proof}
The acceptance criterion at step~6 reads $\delta(\mathbf{P}_{\sigma^*_{k+1}}, \mathbf{R}) \leq \tau = 0$, equivalent to $\|[\mathbf{P}_{\sigma^*_{k+1}}, \mathbf{R}]\|_F = 0$, equivalent to $\mathbf{P}_{\sigma^*_{k+1}} \mathbf{R} = \mathbf{R} \mathbf{P}_{\sigma^*_{k+1}}$, the defining condition $\sigma^*_{k+1} \in \mathrm{Aut}(\mathbf{R})$ from~\eqref{eq:autR}. Because $\mathrm{Aut}(\mathbf{R})$ is closed under composition and inversion, it contains the subgroup $G_K = \langle\sigma^*_1,\ldots,\sigma^*_K\rangle$ generated by the accepted elements.
\end{proof}

\paragraph{Soundness vs.\ completeness.}
The inclusion~\eqref{eq:gen_conv} is one-sided: every accepted permutation is a genuine automorphism of $\mathbf{R}$ (\emph{soundness}), but the procedure may terminate with $G_K$ a proper subgroup of $\mathrm{Aut}(\mathbf{R})$ (\emph{completeness} is not guaranteed). The gap arises because the rounding step in line~5 of Algorithm~\ref{alg:seqgevp} operates on the deflation residual $\mathbf{A}^*_{k+1}$ rather than directly on a candidate permutation matrix. When $\mathrm{Aut}(\mathbf{R}) = S_M$, every Hungarian-rounded permutation lies in $\mathrm{Aut}(\mathbf{R})$ and the rejection test never fires; the procedure recovers all of $\mathrm{Aut}(\mathbf{R})$ trivially. When $\mathrm{Aut}(\mathbf{R})$ is a proper subgroup of $S_M$, however, the deflation residual at iteration $k+1$ may round to a permutation outside $\mathrm{Aut}(\mathbf{R})$, terminating the procedure prematurely. A representative trace appears as Example~\ref{ex:c6} below. Section~\ref{sec:named_families} gives, for five principal named permutation group families ($\mathbb{Z}_n$, $D_n$, $S_K^P$, $S_K \wr S_P$, and the iterated wreath product $G_L$), canonical generator sets whose direct testing achieves both soundness and completeness when the family and its coordinates are known.

\paragraph{Degenerate null spaces and the eigenvector sign.}
Algorithm~\ref{alg:seqgevp} specifies $\mathbf{A}^*_{k+1}$ uniquely only when the smallest eigenvalue of the deflated DC-GEVP is simple and a sign convention is fixed. Neither holds automatically. Whenever two or more linearly independent elements of the deflated span commute with $\mathbf{R}$, the null space has dimension two or more, every vector in it attains $\lambda_{\min} = 0$, and the eigensolver may return any of them. Different null vectors, and the two signs of the same vector, can round to different permutations, some inside $\mathrm{Aut}(\mathbf{R})$ and some outside. The outcome of the plain procedure is then implementation-dependent, as Example~\ref{ex:c6} shows. The four correctness results above are unaffected, since they hold for every admissible choice. Completeness is the property that is sensitive to the choice. Rounding on the entrywise magnitudes $|\mathbf{A}^*_{k+1}|$ removes the dependence on sign, at the price that Lemma~\ref{lem:forward} no longer excludes a repeated element, so membership in $G_k$ must then be tested explicitly.

\begin{example}[Partial recovery on the 6-cycle]\label{ex:c6}
Take $\mathbf{R} = (\mathbf{I} + \mathbf{L}_{C_6})^{-1}$, the diffusion covariance of the cycle $C_6$, for which $\mathrm{Aut}(\mathbf{R}) = D_6$ of order~$12$. Run Algorithm~\ref{alg:seqgevp} at threshold $\tau = 0$ with basis
\[
\mathcal{B} = \{\mathbf{P}_\tau - \mathbf{I},\; \mathbf{P}_{\tau^2} - \mathbf{I},\; \mathbf{P}_\rho - \mathbf{I},\; \mathbf{P}_\eta - \mathbf{I}\},
\]
where $\tau = (1\,2\,3\,4\,5\,6)$ is the cyclic shift, $\tau^2 = (1\,3\,5)(2\,4\,6)$, $\rho = (1\,6)(2\,5)(3\,4)$ is a reflection in $D_6$, and $\eta = (1\,3)$ is a transposition outside $\mathrm{Aut}(\mathbf{R})$. At iteration~1 three basis elements ($\mathbf{P}_\tau - \mathbf{I}$, $\mathbf{P}_{\tau^2} - \mathbf{I}$ and $\mathbf{P}_\rho - \mathbf{I}$) commute with $\mathbf{R}$, so the GEVP returns $\lambda_{\min} = 0$ with a three-dimensional null space, and the rounded permutation depends on which null vector, and which sign, the eigensolver returns. One admissible execution rounds to $\sigma^*_1 = \tau$, accepted, giving $G_1 = \langle\tau\rangle$ of order~$6$. At iteration~2 the deflated null space is spanned by the residual of $\mathbf{P}_\rho - \mathbf{I}$, which is $\mathbf{P}_\rho - \tfrac{1}{3}(\mathbf{P}_\tau + \mathbf{P}_{\tau^3} + \mathbf{P}_{\tau^5})$. With one sign this residual rounds to $\rho$ and the run recovers all of $D_6$. With the other sign it rounds to a permutation outside $\mathrm{Aut}(\mathbf{R})$, the acceptance test rejects, and the run terminates with $G_K = \langle\tau\rangle \subsetneq D_6$. In a numerical study that draws the null vector and its sign at random (script \texttt{verify/check\_seqgevp\_c6\_v1.py} in the source bundle), runs terminate at subgroups of order $1$, $2$, $3$, $6$ and $12$. With sign-free rounding, in which the assignment is computed on the entrywise magnitudes $|\mathbf{A}^*_{k+1}|$, the same study of $2000$ runs ends at order $12$ in most runs and at order $1$ in the remainder. For the execution that stops at $\langle\tau\rangle$, the four named correctness results all hold: Forward Progress gives $\sigma^*_1 = \tau \notin G_0$; Strict Subgroup Growth gives $|G_1| = 6 > 1$; the Iteration Bound reads $K = 1 \leq \lceil\log_2 6\rceil = 3$; Generic Convergence gives $G_K = \langle\tau\rangle \subseteq D_6$. The recovered subgroup is genuine but proper.
\end{example}

The example illustrates a basis-design failure mode: at iteration $k+1$, the deflation residual $\mathbf{A}^*_{k+1}$ may have $\delta(\mathbf{A}^*_{k+1}, \mathbf{R}) = 0$ on the deflated search subspace, yet Hungarian rounding maps it to a permutation outside $\mathrm{Aut}(\mathbf{R})$, prematurely terminating the procedure. Section~\ref{sec:named_families} below removes the failure mode for the principal named permutation group families by testing canonical generators directly, without rounding. The complete picture of what is and is not recoverable from $\mathbf{R}$ alone, the role of the basis design within that picture, and the four open problems (O1)--(O4) that remain after Theorems~\ref{thm:lattice_insens} and~\ref{thm:gen_id} of the Discussion are taken up in Section~\ref{sec:discussion} below.

\subsection{Polynomial-Time Stage C for the Named Families}\label{sec:named_families}

The basis-design failure mode of Example~\ref{ex:c6} is not present for five principal named permutation group families that arise in practice: cyclic $\mathbb{Z}_n$, dihedral $D_n$, direct product $S_K^P$, flat wreath product $S_K \wr S_P$, and iterated wreath product $G_L = (\cdots((S_{K_1} \wr S_{K_2}) \wr S_{K_3}) \cdots) \wr S_{K_L}$ acting on the $\prod_{d=1}^L K_d$ leaves of a regular tree of branching $(K_1, \ldots, K_L)$. For each family, presented in known coordinates, there is a canonical generator set $\mathcal{C}_G$ such that (i)~every element of $\mathcal{C}_G$ lies in $G$, (ii)~$\langle \mathcal{C}_G\rangle = G$, and consequently (iii)~for every proper subgroup $H \subsetneq G$, at least one element of $\mathcal{C}_G$ lies in $G \setminus H$. A procedure that tests each canonical element directly with the exact criterion $\delta = 0$ therefore cannot terminate at a proper subgroup of $G$ when $\mathrm{Aut}(\mathbf{R}) = G$. We summarize the canonical sets and their sizes here; full proofs of the size and coverage properties are recorded in Appendix~\ref{app:canonical_proofs}, with the iterated wreath construction and its inductive proof in Appendix~\ref{app:iw}.

\begin{definition}[Canonical generator sets for the named families]\label{def:canonical_sets}
Let $\tau$ denote the unit cyclic shift on $\{0,\ldots,n-1\}$ and $\rho$ the standard reflection $\rho(i) = (-i) \bmod n$. Let $p(\cdot)$ denote the integer-partition function. The canonical representative of a conjugacy class of $S_K$ is the permutation of that cycle type whose cycles occupy consecutive integers in order of decreasing length, so the representatives of the transposition class and of the $K$-cycle class are $(0\,1)$ and $(0\,1\,\cdots\,K-1)$.
\begin{enumerate}
\item[\textup{(C-Z)}] \emph{Cyclic} $G = \mathbb{Z}_n$: $\mathcal{C}^{\mathbb{Z}}_n = \{\tau\}$. Size $1$.
\item[\textup{(C-D)}] \emph{Dihedral} $G = D_n$: $\mathcal{C}^{D}_n = \{\tau,\,\rho\}$. Size $2$.
\item[\textup{(C-DP)}] \emph{Direct product} $G = S_K^P$: $\mathcal{C}^{\mathrm{DP}}_{K,P}$ contains, for each block $a \in \{0,\ldots,P-1\}$ and each non-identity conjugacy class $[\sigma]$ of $S_K$, the canonical class representative $\sigma$ acting in block $a$ with the identity in all other blocks. Size $P(p(K)-1)$.
\item[\textup{(C-W)}] \emph{Wreath product} $G = S_K \wr S_P$: $\mathcal{C}_{K,P}$ is the union of three candidate types: \textup{(T1)} within-block class representatives as in (C-DP); \textup{(T2)} pure block permutations indexed by non-identity conjugacy classes of $S_P$; and \textup{(T3)} pure block transpositions indexed by unordered pairs $\{a,b\} \subset \{0,\ldots,P-1\}$ with $a \neq b$. Size at most $P(p(K)-1) + (p(P)-1) + \binom{P}{2}$.
\item[\textup{(C-IW)}] \emph{Iterated wreath product} $G_L = (\cdots((S_{K_1} \wr S_{K_2}) \wr S_{K_3}) \cdots) \wr S_{K_L}$ acting on $M_L = \prod_{d=1}^L K_d$ leaves of the regular tree of branching $(K_1, \ldots, K_L)$: $\mathcal{C}^{\mathrm{IW}}_{K_1,\ldots,K_L}$ contains, for each depth $d \in \{1,\ldots,L\}$ and each non-identity conjugacy class of $S_{K_d}$, the canonical class representative acting on the $K_d$ children of the reference depth-$d$ subtree (top-coordinates $(i_L,\ldots,i_{d+1}) = (0,\ldots,0)$) with the identity elsewhere. Size $\sum_{d=1}^{L} (p(K_d)-1)$.
\end{enumerate}
Because $p(K)$ grows faster than every polynomial in $K$, each set that uses class representatives also has a \emph{reduced} form, in which the representatives of a symmetric group $S_K$ are replaced by the two generators $(0\,1)$ and $(0\,1\,\cdots\,K-1)$ alone (one generator when $K = 2$). The reduced sets have sizes at most $2P$ for (C-DP), $2P + 2$ for (C-W) with (T3) omitted, and $2L$ for (C-IW).
\end{definition}

The iterated wreath family (C-IW) contains the symmetric group $S_K$ as the single-level instance $L = 1$, $K_1 = K$, and the flat wreath product as the two-level instance $S_K \wr S_P = G_2$ with $K_1 = K$ and $K_2 = P$, for which it provides an alternative to (C-W). The iterated construction at $L = 2$ yields a strictly smaller canonical set than the flat (C-W) construction; we record this size comparison after the main theorem.

\begin{theorem}[Polynomial-time canonical coverage for the named families]\label{thm:canonical_coverage_dc}
For each family in Definition~\ref{def:canonical_sets} (\textup{(C-Z)}, \textup{(C-D)}, \textup{(C-DP)}, \textup{(C-W)}, \textup{(C-IW)}), the canonical generator set and its reduced form satisfy $\langle \mathcal{C}_G\rangle = G$, and the reduced form has size polynomial in the family parameters. Consequently, for every proper subgroup $H \subsetneq G$, there exists $c \in \mathcal{C}_G$ with $\langle H, c\rangle \supsetneq H$.
\end{theorem}

The size claims follow by direct counting. The proper-subgroup-extension property is immediate from generation, since a proper subgroup cannot contain a generating set. Coverage for (C-Z), (C-D), (C-DP), and (C-W) is proved in Appendix~\ref{app:canonical_proofs} via the standard fact that the transposition $(0\,1)$ and the $K$-cycle $(0\,1\,\cdots\,K-1)$ generate $S_K$ (for the within-block component of (C-DP) and (C-W)) together with case analysis on the block structure (for (C-W) and (C-D)). The cyclic case (C-Z) follows from the structure theorem for cyclic groups; the dihedral case (C-D) from the standard presentation $D_n = \langle\tau,\rho \mid \tau^n = e,\rho^2 = e,\rho\tau\rho = \tau^{-1}\rangle$. The iterated wreath case (C-IW) is proved by induction on $L$ using the semidirect-product factorization $G_L = G_{L-1}^{K_L} \rtimes S_{K_L}$; the inductive step is given in full in Appendix~\ref{app:iw}.

\begin{corollary}[Completeness of canonical-candidate testing for the named families]\label{cor:seqgevp_named}
Let $G$ be a named family of Definition~\ref{def:canonical_sets}, presented in known coordinates, and replace steps~4 and~5 of Algorithm~\ref{alg:seqgevp} by the following candidate test: at iteration $k$, take as $\sigma^*_{k+1}$ any $c \in \mathcal{C}_G$ with $c \notin G_k$ and $\delta(\mathbf{P}_c, \mathbf{R}) = 0$, and terminate when no such $c$ exists. Then the procedure returns $G_K = \langle \mathcal{C}_G \cap \mathrm{Aut}(\mathbf{R})\rangle \subseteq \mathrm{Aut}(\mathbf{R})$, using at most $|\mathcal{C}_G|$ tests per iteration. If $\mathrm{Aut}(\mathbf{R}) = G$, it returns $G_K = G$, achieving both soundness and completeness. The plain form of Algorithm~\ref{alg:seqgevp} with a basis built from $\mathcal{C}_G$ does not carry this guarantee, because every basis element then commutes with $\mathbf{R}$ and the DC-GEVP null space is degenerate.
\end{corollary}

\begin{proof}
Each accepted $c$ satisfies $\delta(\mathbf{P}_c, \mathbf{R}) = 0$, so soundness is Theorem~\ref{thm:gen_conv}. At termination every $c \in \mathcal{C}_G \cap \mathrm{Aut}(\mathbf{R})$ lies in $G_K$, and $G_K$ is generated by accepted elements of that set, so $G_K = \langle \mathcal{C}_G \cap \mathrm{Aut}(\mathbf{R})\rangle$. If $\mathrm{Aut}(\mathbf{R}) = G$, then $\mathcal{C}_G \subseteq \mathrm{Aut}(\mathbf{R})$ and $G_K = \langle \mathcal{C}_G\rangle = G$ by Theorem~\ref{thm:canonical_coverage_dc}.
\end{proof}

\begin{remark}[The 6-cycle revisited]\label{rem:c6_revisited}
The basis of Example~\ref{ex:c6} contains both canonical dihedral generators $\tau$ and $\rho$ of $\mathcal{C}^{D}_6 = \{\tau,\rho\}$. Its failure does not come from the non-automorphism $\eta = (1\,3)$, which never attains $\lambda_{\min} = 0$, but from the degenerate null space and the sign of the returned eigenvector. The same numerical study run with the canonical basis $\{\mathbf{P}_\tau - \mathbf{I},\,\mathbf{P}_\rho - \mathbf{I}\}$ alone, in which every element commutes with $\mathbf{R}$, also terminates at subgroups of several orders ($1$, $2$ and $12$ in the $2000$ runs) under the rounding rule of Algorithm~\ref{alg:seqgevp}, depending on the null vector and its sign. Under sign-free rounding on $|\mathbf{A}^*_{k+1}|$ the canonical basis recovered $D_6$ in all $2000$ runs of the study, while the basis of Example~\ref{ex:c6} did not, so with that rounding rule the extra basis elements are what cause the early stops. The candidate test of Corollary~\ref{cor:seqgevp_named} removes the dependence: $\delta(\mathbf{P}_\tau, \mathbf{R}) = \delta(\mathbf{P}_\rho, \mathbf{R}) = 0$, both are accepted, and $\langle\tau,\rho\rangle = D_6$ is recovered. The failure mode of Example~\ref{ex:c6} therefore depends on both the rounding rule and the basis, and the candidate test avoids it by not rounding at all.
\end{remark}

\begin{remark}[Closed-form Reynolds projectors for the named families]\label{rem:reynolds_closed_form}
Each named family additionally admits a closed-form $O(M^2)$ Reynolds projector $\mathcal{P}_G$ usable in any commutant-restricted estimation step downstream of Algorithm~\ref{alg:seqgevp}. The cyclic projector produces a circulant matrix with entries $c_k = \tfrac{1}{n}\sum_i \widehat{\mathbf{R}}_{i,(i+k) \bmod n}$ for $k = 0,\ldots,n-1$, recovering the standard empirical autocovariance at lags $0,\ldots,n-1$. The dihedral projector produces a symmetric circulant matrix indexed by the folded lag $\min(k, n-k)$, with $\lfloor n/2\rfloor + 1$ orbit averages. The direct-product projector preserves $P(P+1)$ orbit averages indexed by ordered block pairs and by within-block diagonal/off-diagonal type, retaining per-block heterogeneity. The flat wreath projector preserves three orbit averages indexed only by within-block diagonal, within-block off-diagonal, and cross-block type, collapsing block heterogeneity to its common average. The iterated wreath projector preserves $L+1$ orbit averages indexed by the depth $d \in \{0, 1, \ldots, L\}$ of the deepest common ancestor of the index pair in the regular tree, generalizing the flat-wreath three-class structure to $L+1$ depth classes. Closed forms and proofs are recorded in Appendix~\ref{app:reynolds}.
\end{remark}

\begin{remark}[Application-domain coverage]
The five named families cover the major signal-processing application domains where the algebraic-diversity framework is most natural: $\mathbb{Z}_n$ for stationary discrete-time signals, $D_n$ for reflection-symmetric stationary signals, $S_K^P$ for heterogeneous block-symmetric signals (for example, factorial designs with distinct treatment groups), $S_K \wr S_P$ for homogeneous block-exchangeable signals (for example, repeated-measures designs), and the iterated wreath product $G_L$ for hierarchically block-exchangeable signals (for example, some forms of nested experimental designs, multi-level repeated-measures studies, and hierarchical sensor arrays in which sub-arrays are themselves grouped into super-arrays). For a target group $G^* = \mathrm{Aut}(\mathbf{R})$ in any of these families, presented in known coordinates, the candidate test of Corollary~\ref{cor:seqgevp_named} with the reduced canonical set of Definition~\ref{def:canonical_sets} achieves both soundness and completeness in polynomial time.
\end{remark}

\begin{remark}[Iterated versus flat wreath at $L = 2$]\label{rem:iw_vs_flat}
The iterated construction (C-IW) at $L = 2$ with $(K_1, K_2) = (K, P)$ produces a canonical set of total size $(p(K) - 1) + (p(P) - 1)$, strictly smaller than the flat (C-W) bound $P(p(K) - 1) + (p(P) - 1) + \binom{P}{2}$ on the same group $S_K \wr S_P$. The within-block class representatives need only be placed in \emph{one} reference block rather than in every block, and the cross-block transpositions (T3) are redundant once the block-permutation generators (T2) are present. The full proof of equivalence (both canonical sets generate the same group) is given in Appendix~\ref{app:iw_vs_flat}.
\end{remark}

\section{Connections to Known Problems}\label{sec:connections}

\subsection{Relationship to JADE}

The JADE algorithm~\cite{cardoso1993jade} for independent component analysis seeks a unitary matrix $\mathbf{U}$ that jointly diagonalizes a set of fourth-order cumulant matrices $\{\mathbf{K}_1, \ldots, \mathbf{K}_p\}$. This can be expressed as minimizing $\sum_{k=1}^p \|\text{off}(\mathbf{U}^* \mathbf{K}_k \mathbf{U})\|_F^2$, where $\text{off}(\cdot)$ zeros the diagonal.

The double-commutator formulation addresses the same structural question, finding a transform that diagonalizes a matrix, but differs in three key respects: (i)~it operates on the covariance (second order) rather than cumulants (fourth order), requiring weaker distributional assumptions; (ii)~it searches over group generators rather than arbitrary unitaries, constraining the solution to have algebraic structure; and (iii)~the minimum eigenvalue provides a certifiable gap that JADE's iterative rotation lacks.

\subsection{Relationship to Structured Matrix Nearness}

Shah and Chandrasekaran~\cite{shah2012group} studied the problem $\min_{\mathbf{S} \in \mathcal{S}_G} \|\mathbf{R} - \mathbf{S}\|_F$, where $\mathcal{S}_G$ is the set of matrices invariant under the action of group $G$. This finds the nearest $G$-invariant matrix to $\mathbf{R}$.

The double-commutator formulation is the \emph{dual} problem: rather than projecting $\mathbf{R}$ onto a fixed group's invariant subspace, it finds the group whose invariant subspace best contains $\mathbf{R}$. The duality is exact: $\delta(G, \mathbf{R}) = 0$ iff $\mathbf{R} \in \mathcal{S}_G$, i.e., iff the projection residual is zero.

\subsection{Relationship to Simultaneous Diagonalization}

For a normal matrix $\mathbf{A}$, the condition $[\mathbf{A}, \mathbf{R}] = \mathbf{0}$ is equivalent to the statement that $\mathbf{A}$ and $\mathbf{R}$ are simultaneously diagonalizable. The double-commutator GEVP finds the element of a linear subspace that is closest to simultaneously diagonalizable with $\mathbf{R}$, measured in the Frobenius norm. This specializes the general simultaneous diagonalization problem~\cite{bunse1993numerical} to the algebraically structured case where the candidate is constrained to a generator basis, and provides the closed-form solution that iterative Jacobi methods lack.

\section{Experimental Validation}\label{sec:experiments}

We validate the DC-GEVP on two problems: graph automorphism identification and blind chirp rate estimation.

\subsection{Graph Automorphism Identification}

The Automorphism Characterization Theorem~\cite{thornton2026ad} states that a permutation $\sigma$ is a graph automorphism if and only if $\delta(\sigma, \mathbf{L}) = 0$, where $\mathbf{L}$ is the graph Laplacian. The single-generator DC-GEVP provides a constructive test: apply Algorithm~\ref{alg:dc} with a generic basis of permutation generators (cyclic shift, reflection, transposition, block swap, and 3-cycle) to the matrix $\mathbf{R} = (\mathbf{I} + \beta\mathbf{A})^3$ with $\beta = 0.5$, where $\mathbf{A}$ is the adjacency matrix. The map $x \mapsto (1 + \beta x)^3$ is injective, so the argument of the Automorphism Characterization Theorem applies with $\mathbf{A}$ in place of $\mathbf{L}$ and the permutations that commute with $\mathbf{R}$ are exactly the graph automorphisms.

Figure~\ref{fig:graph_aut_1} and Figure~\ref{fig:graph_aut_2} show the commutativity residual $\delta$ for five generic generators tested against six graphs spanning trivial to large automorphism groups. In five of the six cases at least one candidate is a graph automorphism and is identified by $\delta = 0$. For $C_6$, $P_6$ and the prism the candidates split into automorphisms ($\delta = 0$) and non-automorphisms ($\delta > 0$) with a clear gap. For $K_4$ and $K_3$ every candidate is an automorphism, since the automorphism group is the full symmetric group, and all five residuals vanish. For the star, with its center at the first vertex, every one of the five candidates moves the center, none is an automorphism, and all five residuals are positive. The DC-GEVP over the span of the five candidates then returns $\lambda_{\min} > 0$, which certifies that the span contains no commuting element. The star is thus an instance of the basis-design dependence discussed in Section~\ref{sec:seqgevp}: the automorphisms exist but lie outside the candidate span. The single-generator test therefore identifies \emph{at least one} graph automorphism whenever the generic candidate set contains one, without any graph-specific knowledge. For the Abelian-symmetry graph in the panel ($P_6$ with $\mathrm{Aut}(\mathcal{G}) = \mathbb{Z}_2$), the discovered generator generates the full automorphism group by powers, so the single-generator output coincides with full recovery. For the non-Abelian-symmetry graphs in which a candidate is identified ($K_4$ with $\mathrm{Aut}(\mathcal{G}) = S_4$ of order $24$, $C_6$ and the triangular prism with $\mathrm{Aut}(\mathcal{G}) = D_6$ of order $12$, and $K_3$ with $\mathrm{Aut}(\mathcal{G}) = S_3$ of order $6$), the discovered generator is one element of an automorphism group whose full description requires multiple non-commuting generators; the single-generator DC-GEVP does not by itself recover the entire automorphism group in these cases. Full recovery of $\mathrm{Aut}(\mathcal{G})$ is the subject of Algorithm~\ref{alg:seqgevp} (Sequential GEVP) of Section~\ref{sec:seqgevp}, with the soundness-without-completeness characterization recorded there as Theorem~\ref{thm:gen_conv} and Example~\ref{ex:c6}.

\begin{figure*}[t]
\centering
\includegraphics[width=0.32\textwidth]{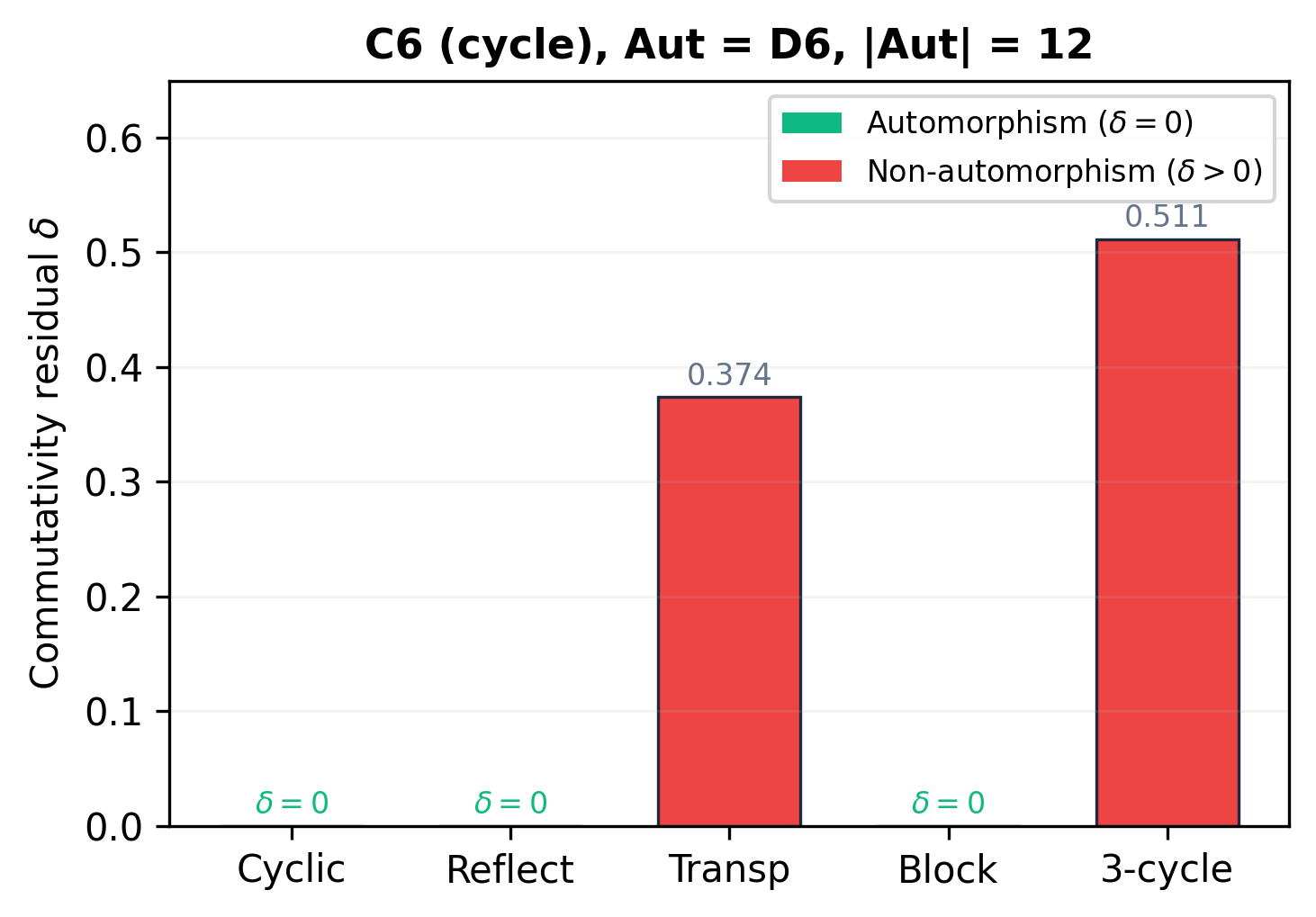}\hfill
\includegraphics[width=0.32\textwidth]{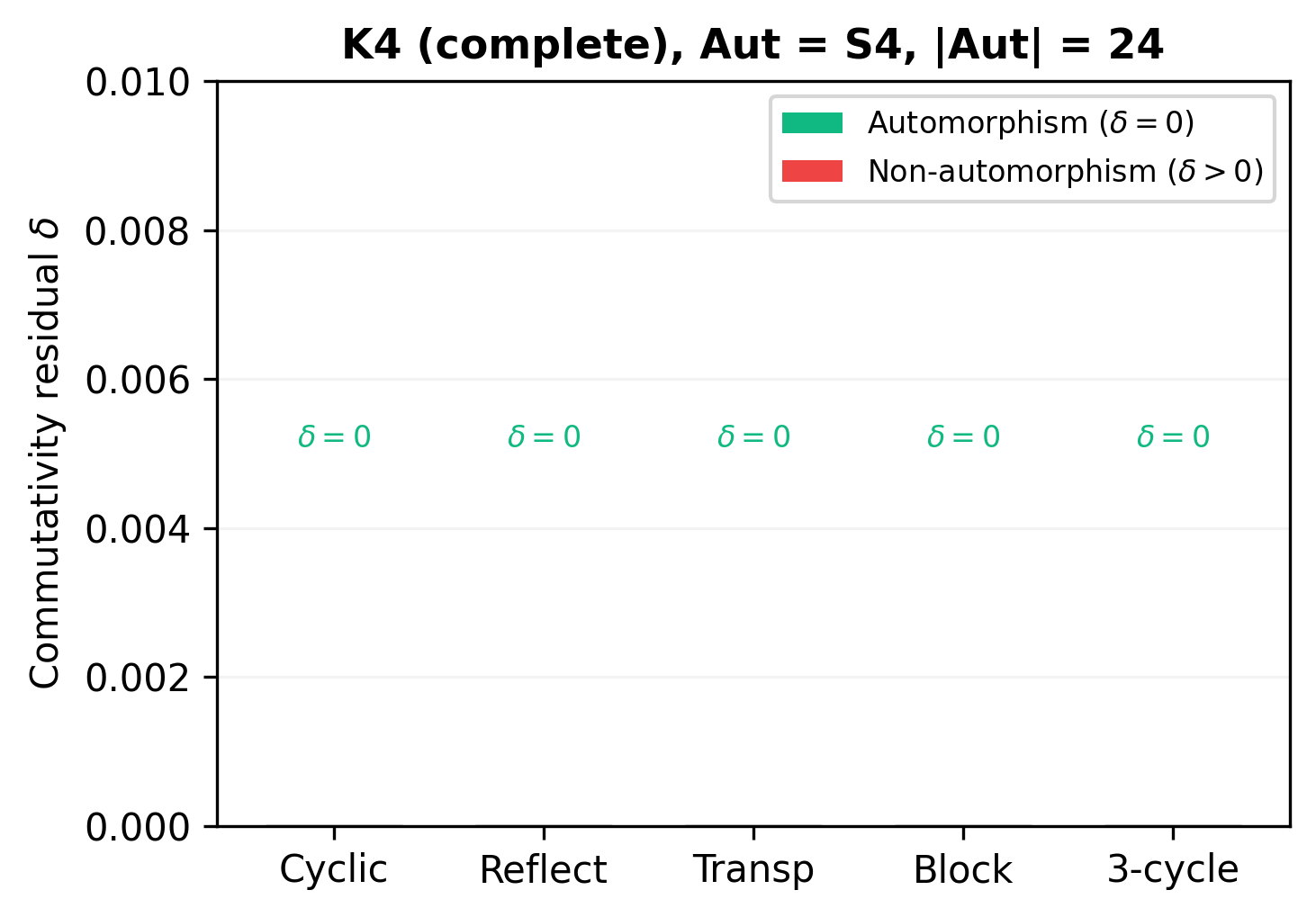}\hfill
\includegraphics[width=0.32\textwidth]{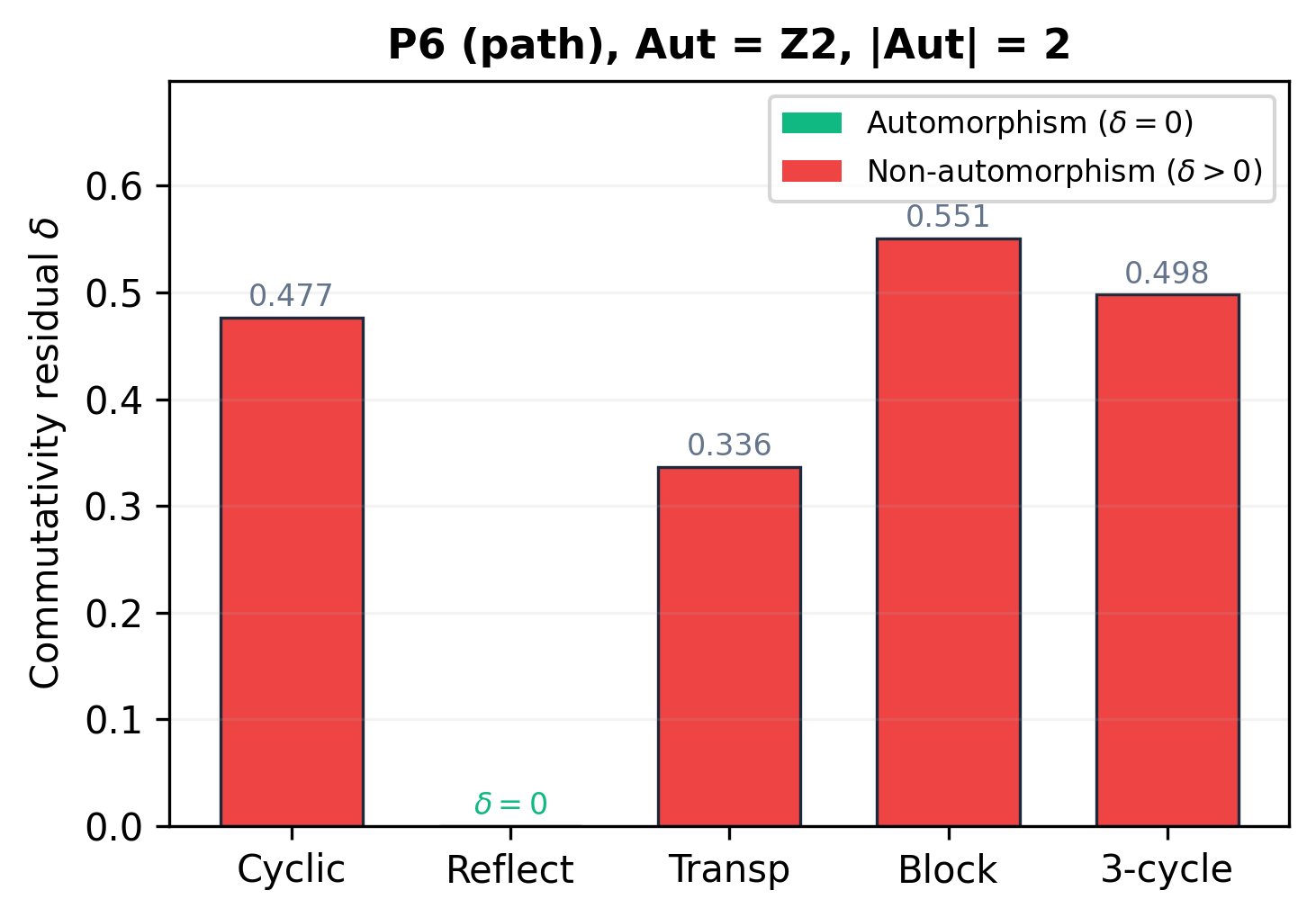}
\caption{Single-generator DC-GEVP graph automorphism identification (Part 1). Commutativity residual $\delta$ for five generic permutation generators on three graphs: cycle $C_6$ ($D_6$, $|\mathrm{Aut}| = 12$), complete $K_4$ ($S_4$, $|\mathrm{Aut}| = 24$), and path $P_6$ ($\mathbb{Z}_2$, $|\mathrm{Aut}| = 2$). Green bars: generators that are graph automorphisms ($\delta = 0$). Red bars: non-automorphisms ($\delta > 0$). The minimum-$\delta$ generator is an automorphism in every case shown, and for $K_4$ all five candidates are automorphisms. No graph-specific information is used. The single-generator DC-GEVP identifies one automorphism per graph; full $\mathrm{Aut}(\mathcal{G})$ recovery for the non-Abelian-symmetry cases ($C_6$ and $K_4$) is the subject of the Sequential GEVP (Algorithm~\ref{alg:seqgevp}, Section~\ref{sec:seqgevp}).}
\label{fig:graph_aut_1}
\end{figure*}

\begin{figure*}[t]
\centering
\includegraphics[width=0.32\textwidth]{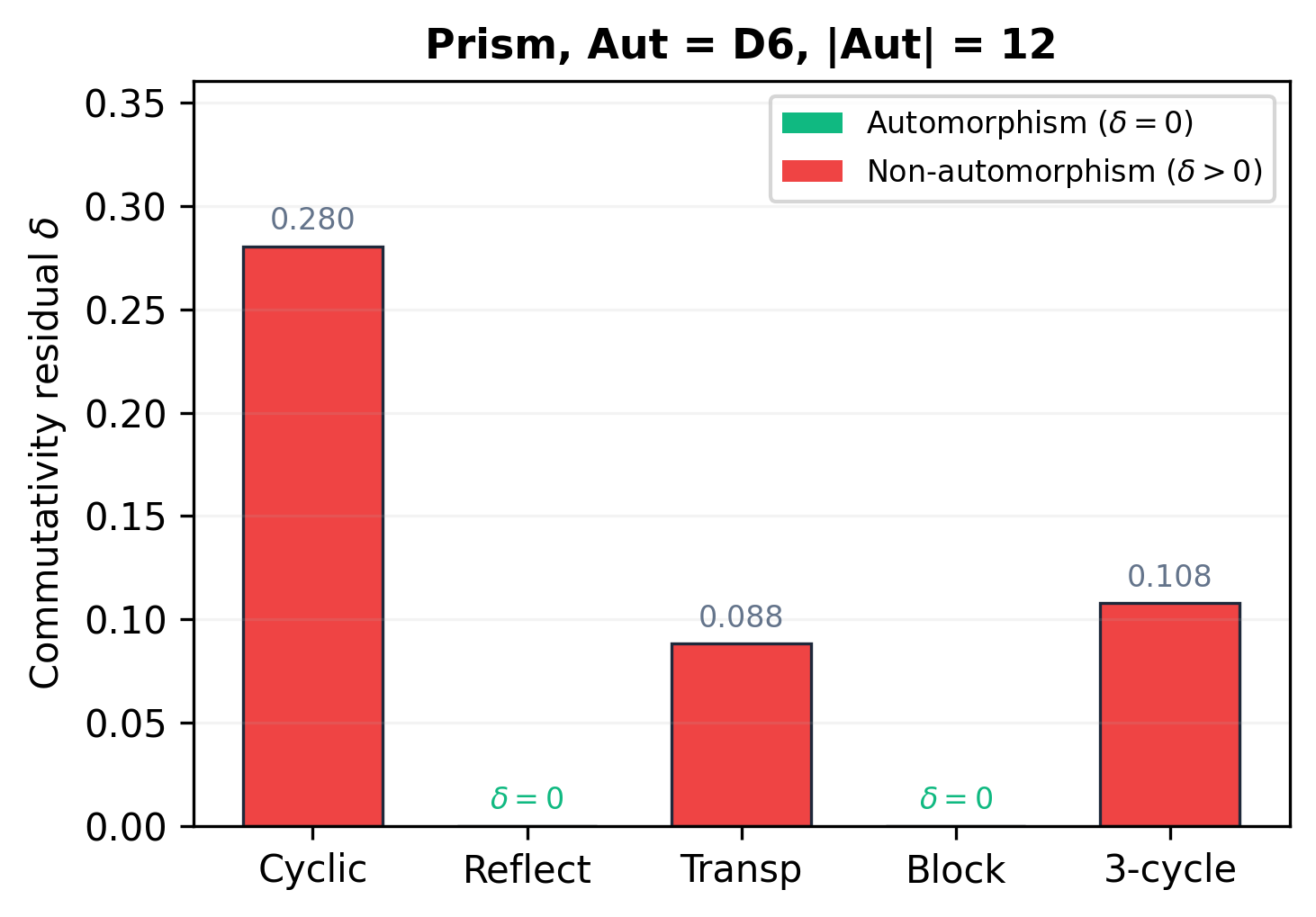}\hfill
\includegraphics[width=0.32\textwidth]{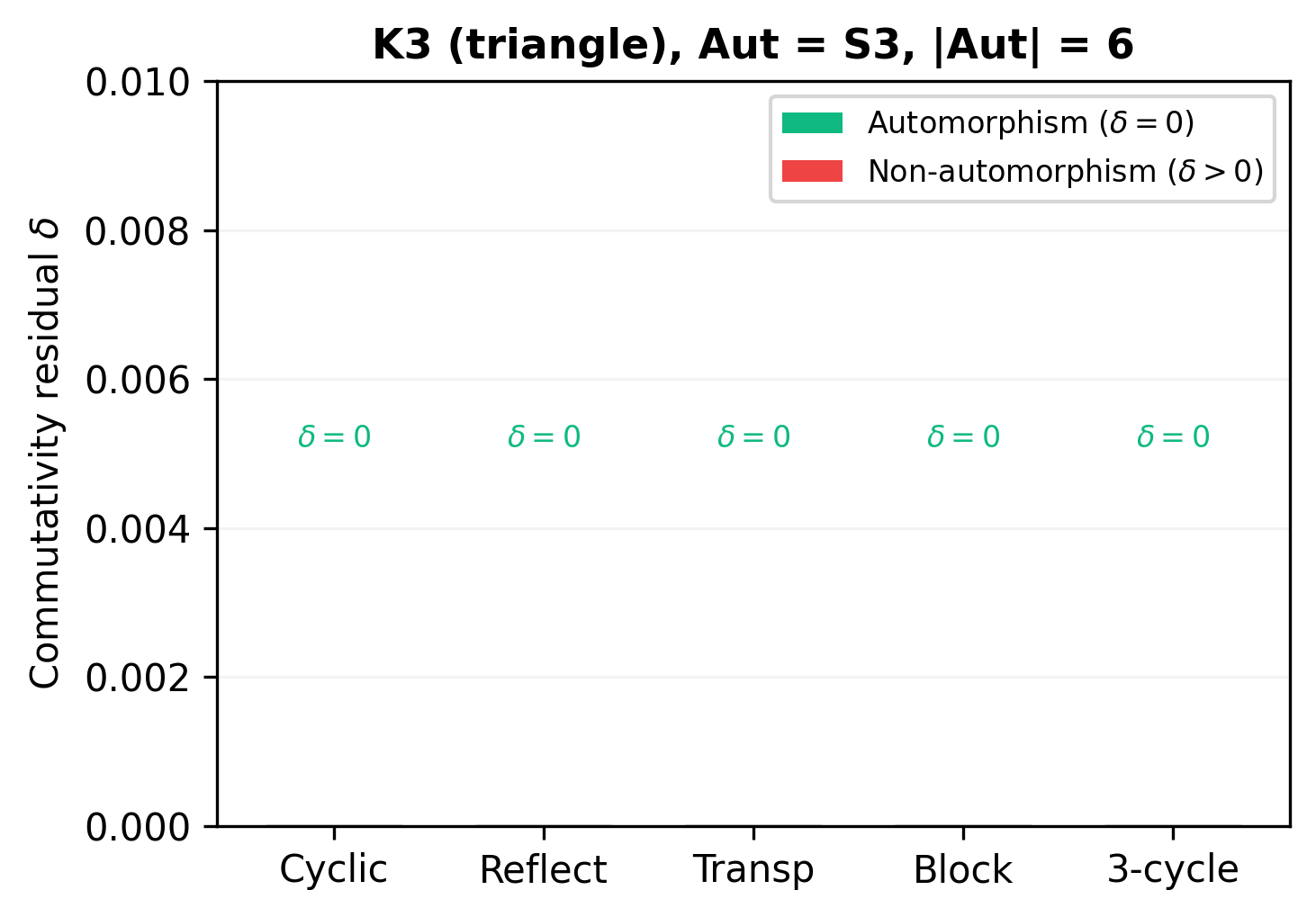}\hfill
\includegraphics[width=0.32\textwidth]{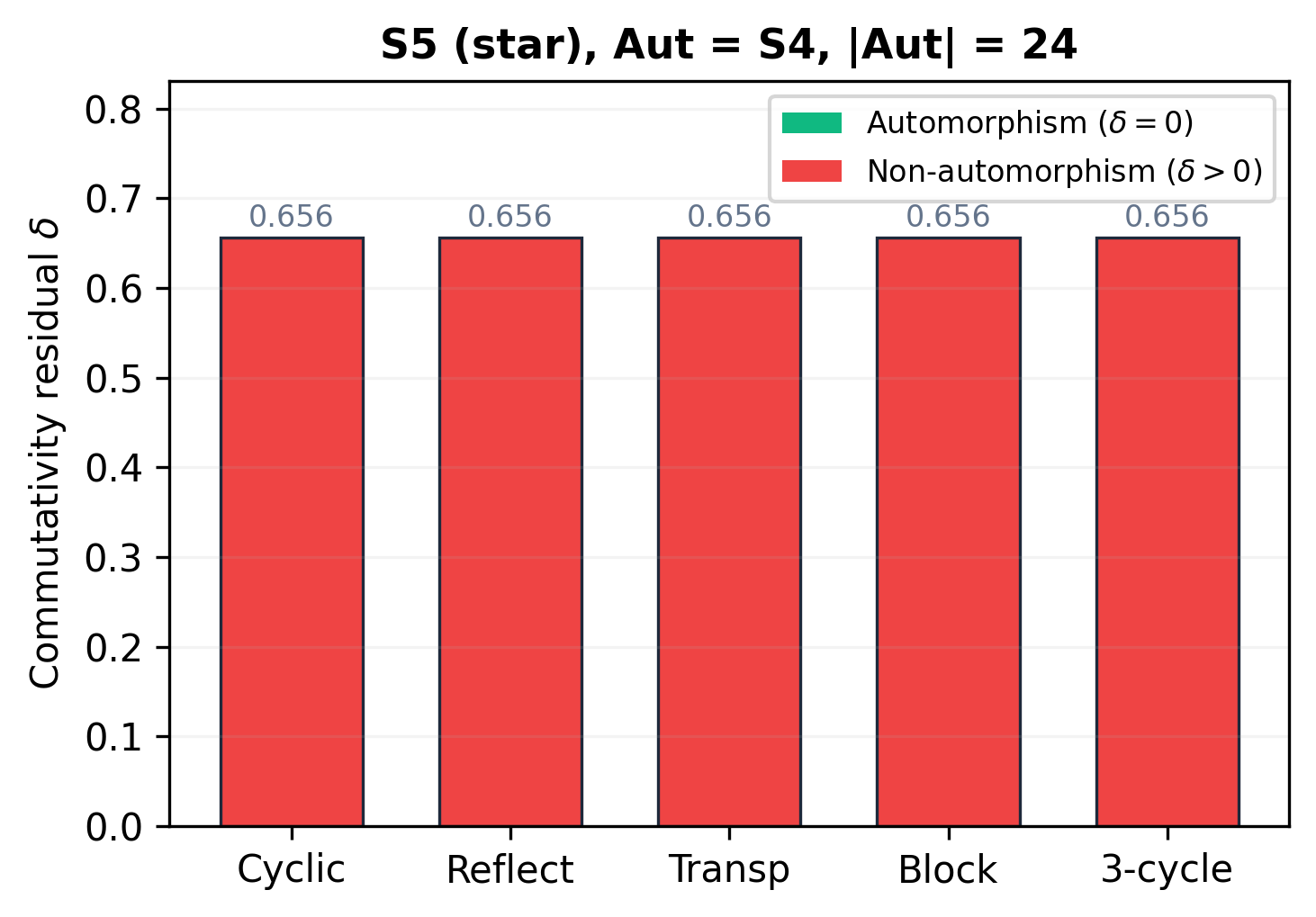}
\caption{Single-generator DC-GEVP graph automorphism identification (Part 2). Triangular prism ($D_6$, $|\mathrm{Aut}| = 12$), complete $K_3$ ($S_3$, $|\mathrm{Aut}| = 6$), and star $S_5$ ($S_4$, $|\mathrm{Aut}| = 24$). For the prism two candidates are automorphisms ($\delta = 0$) and three are not. For $K_3$ all five candidates are automorphisms. For the star, whose center is the first vertex, all five candidates move the center and none is an automorphism, so every residual is positive. The residual is zero exactly for the automorphisms in each panel, in agreement with the Automorphism Characterization Theorem. All three graphs have non-Abelian automorphism groups, so the discovered generator is one element of an $\mathrm{Aut}(\mathcal{G})$ whose full description requires multiple non-commuting generators; full recovery is the subject of Section~\ref{sec:seqgevp}.}
\label{fig:graph_aut_2}
\end{figure*}

\subsection{Blind Chirp Rate Estimation}

For chirp-modulated signals, the matched group is the conjugated cyclic group with chirp parameter $\psi$, written $\mu$ in Proposition~\ref{prop:chirp} (in this subsection $\psi$ is the chirp parameter and not the spectral concentration of Table~\ref{tab:complexity}). The DC-GEVP with a one-dimensional parametric basis $\{B(\psi)\}$ of conjugated cyclic shifts reduces to a sweep of $\lambda_{\min}(\psi) = \|[\mathbf{R}, B(\psi)]\|_F^2 / \|B(\psi)\|_F^2$. Figure~\ref{fig:chirp_sweep} shows the sweep for the population covariance $\mathbf{R} = \mathbf{s}\mathbf{s}^T + \sigma^2\mathbf{I}$ of a real-valued chirp, the real part of $e^{j\pi\psi_0 n^2} e^{j2\pi(0.1)n}$, with $M = 64$ and $\psi_0 = 0.15$. A real chirp contains the mirror component at $-\psi_0$, so the curve is symmetric and has two deep minima of equal depth at $\psi = \pm\psi_0$, which locate $|\psi_0|$. The minima are positive and not zero, because the hypothesis of Proposition~\ref{prop:chirp} holds only approximately here: the signal is real, and its tone completes $6.4$ cycles in $64$ samples and so is off the DFT grid after dechirping. The white-noise term $\sigma^2\mathbf{I}$ commutes with every generator and does not affect the curve, so the SNR in the figure title has no influence on it. For a complex chirp whose dechirped tone is on the DFT grid, the same sweep with $B(\psi) = \mathbf{U}(\psi)^*\mathbf{P}\mathbf{U}(\psi)$ gives a single minimum that vanishes to machine precision at $\psi = \psi_0$, as the proposition states (script \texttt{verify/check\_chirp\_figure\_v1.py} in the source bundle).

\begin{figure}[t]
\centering
\includegraphics[width=\columnwidth]{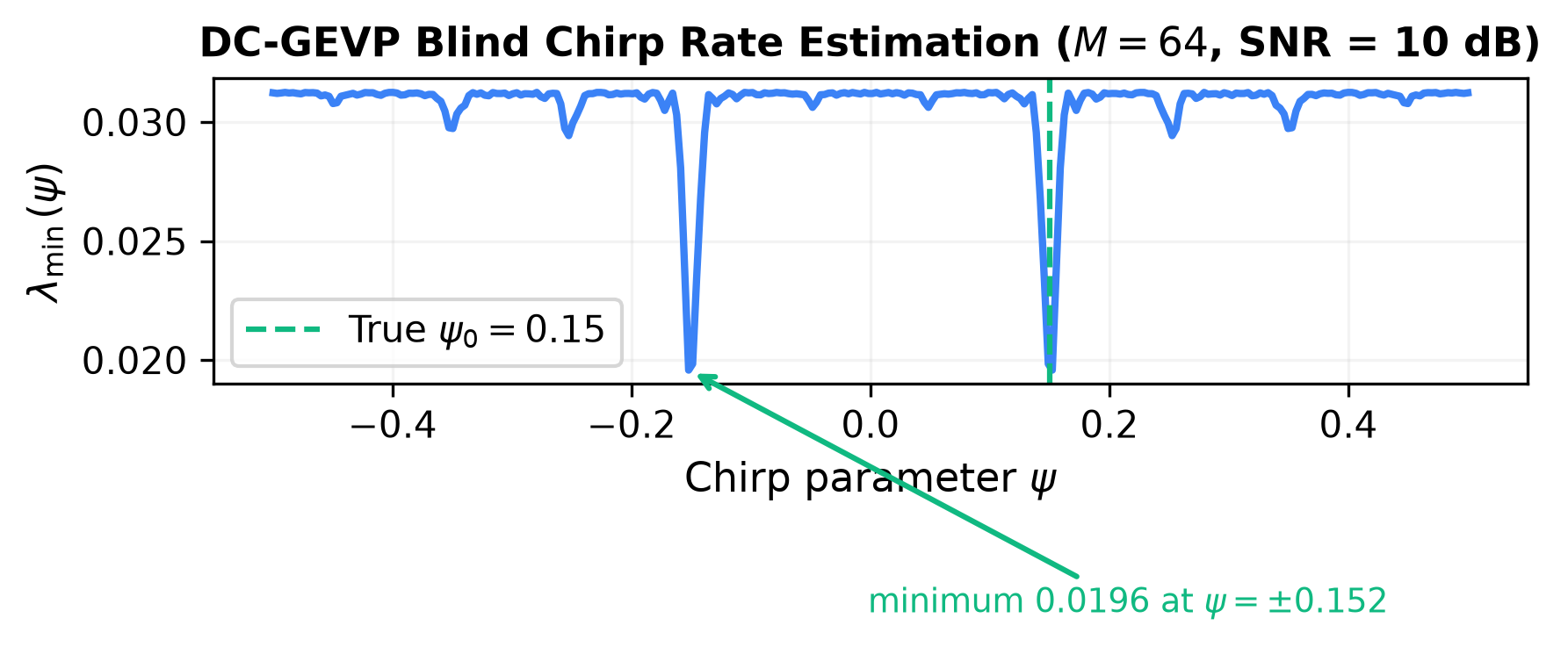}
\caption{DC-GEVP chirp-rate sweep on the population covariance of a real-valued chirp ($M = 64$, $\psi_0 = 0.15$). The minimum eigenvalue $\lambda_{\min}(\psi)$ is plotted against the chirp parameter $\psi$. Two deep minima of equal depth appear at $\psi = \pm\psi_0$, the mirror pair of a real chirp. They are positive because the signal is real and its tone is off the DFT grid. The white-noise term does not affect the curve. The sharpness of the minima indicates high sensitivity to chirp rate mismatch.}
\label{fig:chirp_sweep}
\end{figure}

\section{Discussion}\label{sec:discussion}

\subsection{Spectral Relaxation Tightness}\label{sec:disc_tightness}

The double-commutator reduction reveals that optimal group selection, despite its group-theoretic combinatorial appearance, is fundamentally a spectral problem. The Rayleigh quotient~\eqref{eq:rayleigh} transforms a search over the (discrete, combinatorial) lattice of subgroups of $S_M$ into a search over a (continuous, smooth) $d$-dimensional vector space, where the minimum is attained at an eigenvector.

This reduction is reminiscent of spectral relaxations in combinatorial optimization (e.g., the Fiedler vector for graph partitioning, or semidefinite relaxations for MAX-CUT), but with a crucial difference: for the signal classes of primary practical interest (periodic, symmetric, chirp), the relaxation is \emph{tight}, the minimum eigenvalue is exactly zero and the minimum eigenvector exactly constructs the optimal generator (Propositions~\ref{prop:periodic}--\ref{prop:chirp}). Tightness failures (i.e., $\lambda_{\min} > 0$) are themselves informative: they indicate that no generator in the basis exactly matches the covariance structure, and the magnitude of $\lambda_{\min}$ quantifies the irreducible mismatch.

The computational cost of the double-commutator GEVP is dominated by the $O(d^2 M^2)$ matrix assembly, which is computed once per covariance structure. For typical applications where $d = 3$--$10$ (a small catalog of structural generators) and $M = 16$--$4096$ (sensor array or OFDM size), this is a one-time cost that is small relative to an $O(M^3)$ eigendecomposition of the same covariance.

\subsection{Identifiability Within the Commutant Lattice}\label{sec:disc_identifiability}

The Sequential GEVP of Section~\ref{sec:seqgevp} is sound: every accepted permutation belongs to $\mathrm{Aut}(\mathbf{R})$ at threshold $\tau = 0$. It is not, in general, complete: Example~\ref{ex:c6} exhibits a basis configuration in which Algorithm~\ref{alg:seqgevp} terminates with $G_K \subsetneq \mathrm{Aut}(\mathbf{R})$. To delimit this remaining gap and frame the open problems sharply, we record two further results in this subsection. The first (Theorem~\ref{thm:lattice_insens}) characterizes a class of selection criteria that cannot in principle discriminate within a commutant lattice. The second (Theorem~\ref{thm:gen_id}) characterizes when a Reynolds-projected covariance recovers its generative subgroup exactly versus when its commutant is strictly larger.

For a subgroup $G \subseteq S_M$, write $\mathcal{A}_G := \{X \in \mathbb{C}^{M \times M} : \mathbf{P}_g X = X \mathbf{P}_g \text{ for all } g \in G\}$ for the commutant of the permutation representation, and $\mathcal{P}_G : \mathbb{C}^{M \times M} \to \mathcal{A}_G$ for the orthogonal projector onto $\mathcal{A}_G$ in the Frobenius inner product. The projector admits the Reynolds expression
\begin{equation}\label{eq:reynolds}
\mathcal{P}_G(X) \;=\; \frac{1}{|G|}\sum_{g \in G} \mathbf{P}_g\, X\, \mathbf{P}_g^T.
\end{equation}

\begin{theorem}[Commutant-Lattice Insensitivity]\label{thm:lattice_insens}
Let $G_1 \subseteq G_2 \subseteq S_M$ be subgroups, and suppose $\mathbf{R} \in \mathcal{A}_{G_2}$. Then
\begin{enumerate}
\item[(i)] $\mathcal{A}_{G_1} \supseteq \mathcal{A}_{G_2}$, and consequently $\mathbf{R} \in \mathcal{A}_{G_1}$;
\item[(ii)] $\mathcal{P}_{G_1}(\mathbf{R}) = \mathcal{P}_{G_2}(\mathbf{R}) = \mathbf{R}$ identically;
\item[(iii)] any selection criterion $\Phi(G; \mathbf{R})$ that depends on $G$ only through the projection $\mathcal{P}_G(\mathbf{R})$ satisfies $\Phi(G_1; \mathbf{R}) = \Phi(G_2; \mathbf{R})$ identically, hence cannot discriminate $G_1$ from $G_2$.
\end{enumerate}
\end{theorem}

\begin{proof}
(i) The commutant relation is order-reversing under inclusion: a matrix that commutes with every $\mathbf{P}_g$ for $g \in G_2$ also commutes with every $\mathbf{P}_g$ for $g \in G_1 \subseteq G_2$, so $\mathcal{A}_{G_2} \subseteq \mathcal{A}_{G_1}$. Since $\mathbf{R} \in \mathcal{A}_{G_2}$ by hypothesis, $\mathbf{R} \in \mathcal{A}_{G_1}$ as well.

(ii) The Reynolds projector $\mathcal{P}_G$ is the orthogonal projector onto $\mathcal{A}_G$, so $\mathcal{P}_G(X) = X$ exactly when $X \in \mathcal{A}_G$. By (i), $\mathbf{R} \in \mathcal{A}_{G_1}$ and $\mathbf{R} \in \mathcal{A}_{G_2}$, whence $\mathcal{P}_{G_1}(\mathbf{R}) = \mathbf{R}$ and $\mathcal{P}_{G_2}(\mathbf{R}) = \mathbf{R}$.

(iii) If $\Phi(G; \mathbf{R}) = \tilde{\Phi}(\mathcal{P}_G(\mathbf{R}))$ for some $\tilde{\Phi}$, then $\Phi(G_1; \mathbf{R}) = \tilde{\Phi}(\mathcal{P}_{G_1}(\mathbf{R})) = \tilde{\Phi}(\mathbf{R}) = \tilde{\Phi}(\mathcal{P}_{G_2}(\mathbf{R})) = \Phi(G_2; \mathbf{R})$.
\end{proof}

\begin{remark}[Why the commutator residual escapes Theorem~\ref{thm:lattice_insens}]\label{rem:residual_escapes}
The commutativity residual $\delta(\mathbf{A}, \mathbf{R}) = \|[\mathbf{A}, \mathbf{R}]\|_F / (\|\mathbf{A}\|_F \|\mathbf{R}\|_F)$ is not a function of $\mathcal{P}_G(\mathbf{R})$ alone: it depends on $\mathbf{R}$ directly through the commutator $[\mathbf{A}, \mathbf{R}]$, and on the candidate generator $\mathbf{A}$ rather than on the projector $\mathcal{P}_G$. A generator of $G_2 \setminus G_1$ is available as an $\mathbf{A}$-argument to $\delta$ but is not a generator of $G_1$, so $\delta$ can witness the difference between $G_1$ and $G_2$ even though any criterion of the form $\tilde\Phi(\mathcal{P}_G(\mathbf{R}))$ cannot. This asymmetry is precisely why the DC-GEVP of Theorem~\ref{thm:dc_reduction}, and the Sequential GEVP of Section~\ref{sec:seqgevp} built on it, can in principle discriminate within a commutant lattice that projection-based criteria cannot. The remaining gap, between what $\delta$ can witness and what Algorithm~\ref{alg:seqgevp} actually recovers from a given basis $\mathcal{B}$, is the basis-design open problem (O2) recorded below.
\end{remark}

The second result of this subsection isolates a distinct identifiability gap arising at the level of the data itself, rather than at the level of the criterion.

\begin{theorem}[Generative-Model Identifiability Dichotomy]\label{thm:gen_id}
Let $W \in \mathbb{C}^{M \times M}$ be a Hermitian random matrix whose distribution is absolutely continuous on the real vector space of Hermitian matrices, and let $G^* \subseteq S_M$ be a fixed subgroup. Define
\begin{equation}\label{eq:reynolds_construct}
\mathbf{R} \;:=\; \mathcal{P}_{G^*}(W) \;=\; \frac{1}{|G^*|}\sum_{g \in G^*} \mathbf{P}_g\, W\, \mathbf{P}_g^T,
\end{equation}
so that $\mathbf{R} \in \mathcal{A}_{G^*}$ by construction. Then exactly one of the following holds.
\begin{enumerate}
\item[(i)] \textbf{Identifiable case.} For every supergroup $H \supsetneq G^*$ in $S_M$, there exists at least one $G^*$-orbit on $\{1,\ldots,M\}^2$ that is not preserved by $H$. In this case, $\mathrm{Aut}(\mathbf{R}) = G^*$ almost surely in $W$.
\item[(ii)] \textbf{Ambiguous case.} There exists a supergroup $H \supsetneq G^*$ that preserves every $G^*$-orbit on $\{1,\ldots,M\}^2$. In this case, $\mathcal{A}_{G^*} = \mathcal{A}_H$ as subspaces of $\mathbb{C}^{M \times M}$, the construction~\eqref{eq:reynolds_construct} produces the same $\mathbf{R}$ for both groups, and $\mathrm{Aut}(\mathbf{R}) \supseteq H \supsetneq G^*$ deterministically in every realization of $W$.
\end{enumerate}
\end{theorem}

\begin{proof}
By construction, $\mathbf{R} \in \mathcal{A}_{G^*}$. A permutation $\sigma \in S_M$ commutes with $\mathbf{R}$ if and only if $\mathbf{P}_\sigma \mathbf{R} \mathbf{P}_\sigma^T = \mathbf{R}$ entrywise.

\textit{Case (i).} The entries of $\mathbf{R}$ are constant on $G^*$-orbits in $\{1,\ldots,M\}^2$ for every realization of $W$, because $\mathbf{R} \in \mathcal{A}_{G^*}$ is invariant under the diagonal action $g \cdot (i,j) = (g(i), g(j))$. Suppose $\sigma \notin G^*$ commutes with $\mathbf{R}$, and let $H := \langle G^*, \sigma\rangle \supsetneq G^*$. Since $G^*$ trivially preserves its own orbits and orbit-preservation is closed under composition and inversion, $\sigma$ preserves every $G^*$-orbit if and only if $H$ does. By the case-(i) hypothesis, $H$ does not preserve every $G^*$-orbit, so $\sigma$ does not either: there exist $G^*$-orbits $O \neq O'$ in $\{1,\ldots,M\}^2$ and an index pair $(i,j) \in O$ with $(\sigma(i), \sigma(j)) \in O'$. The commutation requirement $\mathbf{R}_{i,j} = \mathbf{R}_{\sigma(i),\sigma(j)}$ then reads $\mathbf{R}|_O = \mathbf{R}|_{O'}$, a single nontrivial real-linear equation in the entries of $W$ (when $O'$ is the transpose of $O$ it states that the common value on $O$ is real), which defines a measure-zero set under the absolutely continuous distribution of $W$. Taking the union over the finitely many $\sigma \in S_M \setminus G^*$ keeps the failure event measure zero, so $\mathrm{Aut}(\mathbf{R}) = G^*$ almost surely in $W$.

\textit{Case (ii).} If $H \supsetneq G^*$ preserves every $G^*$-orbit on $\{1,\ldots,M\}^2$, then any function constant on $G^*$-orbits is automatically constant on $H$-orbits, because each $H$-orbit decomposes into a union of $G^*$-orbits all carrying the same value. Hence $\mathcal{A}_{G^*} = \mathcal{A}_H$ as subspaces of $\mathbb{C}^{M \times M}$, and the Reynolds projection~\eqref{eq:reynolds_construct} yields the same $\mathbf{R}$ for both groups. Therefore $\mathbf{P}_h \mathbf{R} = \mathbf{R} \mathbf{P}_h$ for every $h \in H$, giving $\mathrm{Aut}(\mathbf{R}) \supseteq H$ deterministically.
\end{proof}

\begin{remark}[Real symmetric covariances]\label{rem:real_sym}
For a real symmetric $W$ the entries of $\mathbf{R}$ are also constant on the transpose of each orbit, so Theorem~\ref{thm:gen_id} holds with each $G^*$-orbit on $\{1,\ldots,M\}^2$ replaced by its union with its transpose. The distinction matters in practice. For $G^* = \mathbb{Z}_n$ with $n \geq 3$ the reflection maps the lag-$k$ orbit to the lag-$(n-k)$ orbit, which is its transpose, so a real symmetric circulant always commutes with the reflection and $\mathrm{Aut}(\mathbf{R}) \supseteq D_n$. The cyclic group is therefore not identifiable from a real covariance, while it is identifiable almost surely from a complex Hermitian one.
\end{remark}

\begin{remark}[Connection to the experimental validation of Section~\ref{sec:experiments}]\label{rem:exp_connection}
Theorem~\ref{thm:lattice_insens} explains why the single-generator DC-GEVP experiments of Section~\ref{sec:experiments} (Figures~\ref{fig:graph_aut_1} and~\ref{fig:graph_aut_2}) demonstrate at most that the discovered generator is consistent with $\mathrm{Aut}(\mathcal{G})$, not that the procedure has identified $\mathrm{Aut}(\mathcal{G})$ in full: the underlying matrix $\mathbf{R} = (\mathbf{I} + \beta\mathbf{A})^3$ lies in $\mathcal{A}_{\mathrm{Aut}(\mathcal{G})}$ by the Automorphism Characterization Theorem~\cite{thornton2026ad}, hence in $\mathcal{A}_G$ for every $G \subseteq \mathrm{Aut}(\mathcal{G})$, and any criterion operating only on the Reynolds projection of $\mathbf{R}$ collapses across that subgroup chain. The single-generator DC-GEVP returns one element of $\mathrm{Aut}(\mathcal{G})$ for each graph whose candidate set contains one; recovery of the entire $\mathrm{Aut}(\mathcal{G})$ requires the Sequential GEVP of Section~\ref{sec:seqgevp}, with the soundness-without-completeness limitation recorded as Theorem~\ref{thm:gen_conv} and Example~\ref{ex:c6}.
\end{remark}

\subsection{Open Problems}\label{sec:disc_opens}

Theorems~\ref{thm:lattice_insens} and~\ref{thm:gen_id} delimit what is and is not in principle recoverable from $\mathbf{R}$ alone. Within those bounds, the following four problems remain. Open problem (O2) below has been resolved for five principal named permutation group families by Section~\ref{sec:named_families} and Appendices~\ref{app:canonical_proofs} and~\ref{app:iw}, and the formulation given here documents the residual general-case question.

\begin{itemize}
\item[(O1)] \textbf{Generative-model identifiability.} For each candidate target group $G^* \subseteq S_M$ in a given application, characterize whether case (i) or case (ii) of Theorem~\ref{thm:gen_id} applies. The dichotomy itself is established by Theorem~\ref{thm:gen_id}; what remains open is the explicit combinatorial characterization for the wreath, semidirect, and Cartesian-product families that arise in practice. A combinatorial test of orbit-pair preservation for those families would resolve the question case-by-case and identify whether $G^*$ is recoverable from a single Reynolds-projected $\mathbf{R}$ or whether only a strict supergroup is recoverable.

\item[(O2)] \textbf{Basis-design conditions.} Find sufficient conditions on the basis $\mathcal{B}$ (cardinality, structure, alignment with $\mathrm{Aut}(\mathbf{R})$) under which Algorithm~\ref{alg:seqgevp} at $\tau = 0$ discovers all of $\mathrm{Aut}(\mathbf{R})$ rather than a strict subgroup. Example~\ref{ex:c6} demonstrates that the interaction between deflation and Hungarian rounding can cause the rounded permutation $\sigma^*_{k+1}$ to fall outside $\mathrm{Aut}(\mathbf{R})$ even when the deflation residual $\mathbf{A}^*_{k+1}$ has $\delta(\mathbf{A}^*_{k+1},\mathbf{R}) = 0$, prematurely terminating the procedure. \emph{For five principal named permutation group families}, namely cyclic $\mathbb{Z}_n$, dihedral $D_n$, direct product $S_K^P$, flat wreath product $S_K \wr S_P$, and iterated wreath product $G_L = (\cdots((S_{K_1} \wr S_{K_2}) \wr S_{K_3}) \cdots) \wr S_{K_L}$, \emph{the basis-design problem is resolved by Corollary~\ref{cor:seqgevp_named} and Theorem~\ref{thm:canonical_coverage_dc} (with proofs in Appendices~\ref{app:canonical_proofs} and~\ref{app:iw}): canonical generator sets of total size $1$, $2$, $P(p(K)-1)$, at most $P(p(K)-1) + (p(P)-1) + \binom{P}{2}$, and $\sum_{d=1}^{L} (p(K_d)-1)$ respectively, with reduced two-generator forms of polynomial size, provide candidates whose direct testing at $\tau = 0$ achieves both soundness and completeness when the family and its coordinates are known}. A rule for selecting the null vector and its sign in the plain procedure, when the DC-GEVP null space is degenerate, is a related open point. The general proper-subgroup case beyond these named families remains open, and the resolved catalog is currently growing one structured family at a time rather than via a uniform structural theorem; such a theorem, if available, would replace the case-by-case extension with a single result, and we record the question as open. A Restricted Commutativity Property (RCP) condition on $\mathcal{B}$, by analogy with the Restricted Isometry Property in compressed sensing~\cite{candes2005decoding}, would be of independent mathematical interest and would extend the basis-design resolution to a broader class of automorphism groups.

\item[(O3)] \textbf{Noise robustness at $\tau > 0$.} Characterize the false-positive rate (permutations admitted at threshold $\tau$ that are not in $\mathrm{Aut}(\mathbf{R})$) as a function of $\tau$, the SNR of an additive isotropic perturbation of $\mathbf{R}$, and $|\mathrm{Aut}(\mathbf{R})|$, with a quantitative tail bound. The single-generator DC-GEVP at finite SNR admits a Davis--Kahan-type perturbation analysis on the eigenvector $\mathbf{c}^*$; an analogous analysis for the sequential procedure at threshold $\tau > 0$ would underpin a principled threshold-selection rule.

\item[(O4)] \textbf{Joint multi-generator extension.} Algorithm~\ref{alg:seqgevp} is greedy: it picks one generator at a time and deflates. An alternative is to optimize over multi-dimensional subspaces of $\mathrm{span}(\mathcal{B})$ simultaneously, using the second-smallest GEVP eigenpair, or a tensor generalization that solves for a multi-generator subspace in a single step. Whether such an extension improves recovery of non-Abelian groups in low-SNR regimes, and whether it admits a closed-form analogue of Theorem~\ref{thm:dc_reduction}, are open.
\end{itemize}

\section{Conclusion}\label{sec:conclusion}

We have proved that optimal group selection for algebraic diversity reduces from a combinatorial search over subgroups of $S_M$ to a $d \times d$ generalized eigenvalue problem derived from the double commutator of the covariance matrix. The reduction is polynomial-time, closed-form, certifiable, and exact for the principal structured signal classes. The problem does not appear in the standard complexity catalog and sits at the intersection of group theory, matrix analysis, and statistical estimation. The double-commutator inner product is the unique commutativity measure, up to scale, built from a unitarily invariant quadratic norm of the commutator, and the formulation is related to JADE, structured matrix nearness, and simultaneous diagonalization. For non-Abelian symmetry groups, the Sequential GEVP of Section~\ref{sec:seqgevp} extends the single-generator construction via group-theoretic deflation, with four named correctness results that bound the recovered subgroup and the number of iterations; soundness ($G_K \subseteq \mathrm{Aut}(\mathbf{R})$) is unconditional, while completeness ($G_K = \mathrm{Aut}(\mathbf{R})$) is established when $\mathrm{Aut}(\mathbf{R}) = S_M$ and, via direct testing of the canonical generator sets of Section~\ref{sec:named_families} and Appendices~\ref{app:canonical_proofs} and~\ref{app:iw}, for five principal named permutation group families ($\mathbb{Z}_n$, $D_n$, $S_K^P$, $S_K \wr S_P$, and the iterated wreath product $G_L$) covering the major signal-processing application domains. The general proper-subgroup case beyond these named families remains open (Example~\ref{ex:c6}, open problem (O2)).

Two further results characterize the identifiability landscape within which any group-selection procedure must operate. The Commutant-Lattice Insensitivity Theorem (Theorem~\ref{thm:lattice_insens}) shows that any selection criterion operating only on the Reynolds projection of $\mathbf{R}$ cannot discriminate between $G_1 \subseteq G_2$ when $\mathbf{R} \in \mathcal{A}_{G_2}$; criteria of that form collapse identically across the entire commutant lattice rooted at $\mathbf{R}$. The commutativity residual $\delta(\mathbf{A}, \mathbf{R})$ escapes this insensitivity (Remark~\ref{rem:residual_escapes}), which is the population-level reason the DC-GEVP and its sequential extension can in principle discriminate within a commutant lattice that projection-based criteria cannot. The Generative-Model Identifiability Dichotomy (Theorem~\ref{thm:gen_id}) characterizes, when $\mathbf{R}$ is the Reynolds projection of a random Hermitian $W$ associated to a generative subgroup $G^* \subseteq S_M$, whether $\mathrm{Aut}(\mathbf{R})$ recovers $G^*$ exactly or strictly contains a supergroup, in terms of the orbit-pair structure of $G^*$ on $\{1,\ldots,M\}^2$. Together these two theorems delimit what is and is not in principle recoverable from $\mathbf{R}$ alone, and isolate the residual algorithmic gap into the four open problems (O1)--(O4) of Section~\ref{sec:disc_opens}.

The broader implication is that the algebraic diversity framework, which generalizes temporal averaging to group-structured single-observation estimation~\cite{thornton2026ad,thornton2026framework_arxiv}, has a polynomial-time end-to-end pipeline: both the estimation step (the group-averaged estimator, the constrained maximum likelihood estimator in the Gaussian case~\cite{thornton2026ad}) and the model selection step (optimal generator selection via the DC-GEVP, with multi-generator extension via Algorithm~\ref{alg:seqgevp}) are polynomial-time, making the framework practical for real-time signal processing. The Trivial Group Embedding Theorem~\cite{thornton2026ad} establishes that conventional temporal averaging is the degenerate case of this pipeline with $G = \{e\}$, $d_{\mathrm{eff}} = 1$; the double-commutator GEVP provides the mechanism for moving beyond the trivial group to exploit the structural capacity $\kappa(\mathbf{R}) = (\Tr\mathbf{R})^2/\Tr(\mathbf{R}^2)$ of the signal~\cite{thornton2026framework_arxiv}, with the identifiability limits of Theorems~\ref{thm:lattice_insens} and~\ref{thm:gen_id} fixing the boundary of what any procedure within the framework can in principle achieve.

\appendix
\titleformat{\section}{\normalfont\large\scshape\centering}{\appendixname~\Alph{section}.}{0.5em}{}

\section{Coverage proofs for the named families}\label{app:canonical_proofs}

This appendix proves the size, coverage, and proper-subgroup-extension claims of Theorem~\ref{thm:canonical_coverage_dc} for the four non-iterated named families (C-Z), (C-D), (C-DP), and (C-W). The iterated wreath case (C-IW) is treated separately in Appendix~\ref{app:iw}, with the size comparison at $L = 2$ in Appendix~\ref{app:iw_vs_flat}.

\subsection{Cyclic case (C-Z)}

For $G = \mathbb{Z}_n$ acting on $\{0, \ldots, n-1\}$ by addition modulo $n$, the canonical set is $\mathcal{C}^{\mathbb{Z}}_n = \{\tau\}$ where $\tau$ is the unit shift $i \mapsto (i+1) \bmod n$. Size is $1$, trivially polynomial. Coverage $\langle \tau \rangle = \mathbb{Z}_n$ is the structure theorem for cyclic groups. For the proper-subgroup-extension property, every proper subgroup $H \subsetneq \mathbb{Z}_n$ is generated by some $\tau^{d}$ with $d \mid n$, $d > 1$, and consists of multiples of $d$ modulo $n$. Then $\tau \notin H$ because $1$ is not a multiple of $d$, hence $\langle H, \tau \rangle = \langle \tau \rangle = \mathbb{Z}_n \supsetneq H$.

\subsection{Dihedral case (C-D)}

For $G = D_n$ acting on $\{0, \ldots, n-1\}$ with the standard presentation $D_n = \langle \tau, \rho \mid \tau^n = e, \rho^2 = e, \rho \tau \rho = \tau^{-1} \rangle$, the canonical set is $\mathcal{C}^{D}_n = \{\tau, \rho\}$ where $\rho$ is the reflection $i \mapsto (-i) \bmod n$. Size is $2$. Coverage $\langle \tau, \rho \rangle = D_n$ follows from the presentation. For the proper-subgroup-extension property, the subgroup lattice of $D_n$ consists of: the trivial group, cyclic subgroups $\langle \tau^d \rangle$ for $d \mid n$, dihedral subgroups $\langle \tau^d, \tau^j \rho \rangle$ for $d \mid n$ and $j \in \{0, \ldots, d-1\}$, and the full group. For any proper subgroup $H$: if $H$ is purely cyclic ($H \subseteq \langle \tau \rangle$), then $\rho \notin H$ and $\langle H, \rho \rangle$ contains both rotations and reflections, hence is dihedral and strictly contains $H$. If $H$ is dihedral $\langle \tau^d, \tau^j \rho \rangle$ with $d > 1$, then $\tau \notin H$ (since $\tau$ has order $n$ but $H$ contains only powers of $\tau^d$ of order $\leq n/d$ paired with reflections), and $\langle H, \tau \rangle$ contains all of $\langle \tau \rangle$ together with the reflection $\tau^j \rho \in H$, hence equals $D_n$ which strictly contains $H$.

\subsection{Direct product case (C-DP)}

For $G = S_K^P$ acting on $\{0, \ldots, K-1\} \times \{0, \ldots, P-1\}$ with one copy of $S_K$ acting independently on each block, the canonical set $\mathcal{C}^{\mathrm{DP}}_{K,P}$ contains, for each block $a \in \{0, \ldots, P-1\}$ and each non-identity conjugacy class $[\sigma]$ of $S_K$, the canonical class representative $\sigma$ acting in block $a$ with the identity in all other blocks.

\emph{Size.} The number of conjugacy classes of $S_K$ is the integer-partition function $p(K)$; subtracting the identity class gives $p(K) - 1$ non-identity classes. With $P$ blocks, the size is $P(p(K) - 1)$. This is polynomial in $P$ but not in $K$, and the reduced form has size at most $2P$.

\emph{Coverage.} The standard generating fact for $S_K$ is that the transposition $(0, 1)$ together with the $K$-cycle $(0, 1, 2, \ldots, K-1)$ generate $S_K$; both are canonical class representatives of their respective conjugacy classes (transpositions and $K$-cycles). Hence the canonical class representatives of $S_K$ generate $S_K$. Applied independently in each block, the canonical representatives in $\mathcal{C}^{\mathrm{DP}}_{K,P}$ generate $S_K^P$.

\emph{Proper-subgroup-extension.} If $H \subsetneq S_K^P$ is proper, some $c \in \mathcal{C}^{\mathrm{DP}}_{K,P}$ is not in $H$, since otherwise $H$ would contain a generating set of $S_K^P$. Then $\langle H, c \rangle \supsetneq H$. The same argument applies to the reduced form.

\subsection{Flat wreath case (C-W)}

For $G = S_K \wr S_P$ acting on $\{0, \ldots, K-1\} \times \{0, \ldots, P-1\}$, the canonical set $\mathcal{C}_{K,P}$ is the union of three types: \textup{(T1)} within-block class representatives as in (C-DP); \textup{(T2)} pure block permutations $\widetilde\pi$ for each non-identity conjugacy class representative $\pi$ of $S_P$, where $\widetilde\pi$ permutes blocks according to $\pi$ and acts as the identity within each block; \textup{(T3)} pure block transpositions $\widetilde{\tau}_{a,b}$ for each unordered pair $\{a, b\} \subset \{0, \ldots, P-1\}$, where $\widetilde\tau_{a,b}$ swaps blocks $a$ and $b$ and acts as the identity elsewhere.

\emph{Size.} (T1) contributes $P(p(K)-1)$, (T2) contributes $p(P)-1$, (T3) contributes $\binom{P}{2}$. Total at most $P(p(K)-1) + (p(P)-1) + \binom{P}{2}$, polynomial in $P$ but not in $K$, with a reduced form of size at most $2P + 2$.

\emph{Coverage.} The semidirect product structure $S_K \wr S_P = S_K^P \rtimes S_P$ decomposes every $g \in G$ uniquely as $g = h \widetilde\pi$ with $h \in S_K^P$ and $\widetilde\pi$ a pure block permutation. By coverage of (C-DP), $h$ is generated by (T1). By the same argument applied to $S_P$, $\widetilde\pi$ is generated by canonical class representatives of $S_P$, which are exactly (T2). Hence (T1) $\cup$ (T2) generates $G$, and (T3) is redundant for generation.

\emph{Proper-subgroup-extension.} As in (C-DP), a proper subgroup cannot contain the generating set (T1) $\cup$ (T2), so some canonical element lies outside it.

\section{The iterated wreath case (C-IW)}\label{app:iw}

This appendix treats the iterated wreath family separately because its proof is by induction on $L$ and uses the recursive structure $G_L = G_{L-1}^{K_L} \rtimes S_{K_L}$ in an essential way.

\subsection{Notation}

Let the regular tree of branching $(K_1, K_2, \ldots, K_L)$ have $M_L = \prod_{d=1}^L K_d$ leaves indexed by tuples $(i_L, i_{L-1}, \ldots, i_1)$ with $i_d \in \{0, \ldots, K_d - 1\}$, where $i_L$ is the root-level index and $i_1$ is the leaf-level index. The iterated wreath group $G_L = (\cdots((S_{K_1} \wr S_{K_2}) \wr S_{K_3}) \cdots) \wr S_{K_L}$ acts on the leaf set with the recursion $G_1 = S_{K_1}$ and $G_d = G_{d-1} \wr S_{K_d}$ for $d \geq 2$. The semidirect-product factorization at depth $L$ is $G_L = G_{L-1}^{K_L} \rtimes S_{K_L}$, where the $K_L$ copies of $G_{L-1}$ act on the $K_L$ depth-$(L-1)$ subtrees and $S_{K_L}$ permutes those subtrees.

The canonical set is
\[
\mathcal{C}^{\mathrm{IW}}_{K_1, \ldots, K_L} = \bigcup_{d=1}^L \mathcal{C}^{(d)}, \qquad |\mathcal{C}^{\mathrm{IW}}| = \sum_{d=1}^L (p(K_d) - 1),
\]
where $\mathcal{C}^{(d)}$ contains, for each non-identity conjugacy class of $S_{K_d}$, the canonical class representative acting on the $K_d$ children of the reference depth-$d$ subtree (the subtree with top coordinates $(i_L, i_{L-1}, \ldots, i_{d+1}) = (0, 0, \ldots, 0)$) with the identity elsewhere.

\subsection{Theorem and proof}

\begin{theorem}[Iterated wreath canonical coverage]\label{thm:iw_canonical}
For every $L \geq 1$ and every branching $(K_1, \ldots, K_L)$, the canonical set $\mathcal{C}^{\mathrm{IW}}_{K_1, \ldots, K_L}$ satisfies $\langle \mathcal{C}^{\mathrm{IW}} \rangle = G_L$ and the proper-subgroup-extension property: for every proper subgroup $H \subsetneq G_L$, there exists $c \in \mathcal{C}^{\mathrm{IW}}$ with $\langle H, c \rangle \supsetneq H$. The size is $\sum_{d=1}^L (p(K_d) - 1)$, and the reduced form has size at most $2L$.
\end{theorem}

\begin{proof}
\emph{Base case $L = 1$.} $G_1 = S_{K_1}$ acting on $K_1$ points. $\mathcal{C}^{(1)}$ contains the canonical class representatives of $S_{K_1}$. Coverage and the proper-subgroup-extension property reduce to the fact that the canonical class representatives of $S_K$ generate $S_K$ (used in Appendix~\ref{app:canonical_proofs}, (C-DP)).

\emph{Inductive step.} Assume the theorem holds for $G_{L-1}$. The semidirect-product factorization gives $G_L = G_{L-1}^{K_L} \rtimes S_{K_L}$, where the base $G_{L-1}^{K_L}$ consists of $K_L$ independent copies of $G_{L-1}$ acting on the $K_L$ depth-$(L-1)$ subtrees, and $S_{K_L}$ permutes the subtrees.

The canonical set decomposes as $\mathcal{C}^{\mathrm{IW}}_L = \mathcal{C}^{<L} \cup \mathcal{C}^{(L)}$, where $\mathcal{C}^{<L} = \bigcup_{d=1}^{L-1} \mathcal{C}^{(d)}$ contains the lower-depth generators (each acting on the reference subtree at depth $\leq L-1$ with the identity outside) and $\mathcal{C}^{(L)}$ contains the canonical class representatives of $S_{K_L}$ acting on the $K_L$ depth-$(L-1)$ subtrees.

\emph{Coverage.} Let $G_{L-1}^{(0)}$ denote the copy of $G_{L-1}$ acting on the reference depth-$(L-1)$ subtree (top coordinate $i_L = 0$) and as the identity on all other subtrees. By the inductive hypothesis applied to $G_{L-1}^{(0)}$, the lower-depth generators $\mathcal{C}^{<L}$ generate $G_{L-1}^{(0)}$. By the base-case argument applied to $S_{K_L}$, $\mathcal{C}^{(L)}$ generates the lifted $S_{K_L}^{(\mathrm{block})} \subset G_L$ that permutes the $K_L$ depth-$(L-1)$ subtrees.

The remaining task is to verify that conjugating $G_{L-1}^{(0)}$ by elements of $S_{K_L}^{(\mathrm{block})}$ yields all of $G_{L-1}^{K_L}$. For each $j \in \{0, \ldots, K_L - 1\}$, choose $\widetilde\pi_j \in S_{K_L}^{(\mathrm{block})}$ that maps subtree $0$ to subtree $j$ (such $\widetilde\pi_j$ exists by transitivity of $S_{K_L}$ on $K_L$ points, achieved by $\mathcal{C}^{(L)}$). Then $\widetilde\pi_j G_{L-1}^{(0)} \widetilde\pi_j^{-1} = G_{L-1}^{(j)}$, the copy of $G_{L-1}$ acting on subtree $j$ with the identity elsewhere. Hence the subgroup generated by $\mathcal{C}^{<L} \cup \mathcal{C}^{(L)}$ contains $G_{L-1}^{(j)}$ for every $j$, hence contains $\prod_{j=0}^{K_L - 1} G_{L-1}^{(j)} = G_{L-1}^{K_L}$. Combined with $S_{K_L}^{(\mathrm{block})}$ from $\mathcal{C}^{(L)}$, the semidirect-product decomposition $G_L = G_{L-1}^{K_L} \rtimes S_{K_L}^{(\mathrm{block})}$ yields $\langle \mathcal{C}^{\mathrm{IW}}_L \rangle = G_L$.

\emph{Proper-subgroup-extension.} If $H \subsetneq G_L$ is proper, some $c \in \mathcal{C}^{\mathrm{IW}}_L$ is not in $H$, since otherwise $H$ would contain the generating set just established. Then $\langle H, c \rangle \supsetneq H$.

The size $\sum_{d=1}^L (p(K_d) - 1)$ is immediate from the construction.
\end{proof}

\subsection{Polynomial-time AD construction}\label{app:iw_construction}

The canonical set $\mathcal{C}^{\mathrm{IW}}_{K_1, \ldots, K_L}$ admits a polynomial-time algebraic-diversity construction: for each depth $d \in \{1, \ldots, L\}$ and each non-identity conjugacy class $[\sigma]$ of $S_{K_d}$, lift the canonical class representative $\sigma \in S_{K_d}$ to a permutation matrix on the $M_L$ leaves by acting nontrivially only within the reference depth-$d$ subtree (the $K_d$ children of the subtree with top coordinates $(i_L, \ldots, i_{d+1}) = (0, \ldots, 0)$) and as the identity elsewhere. The total construction time is $O\bigl(\bigl(\sum_{d=1}^L p(K_d)\bigr) M_L^3\bigr)$, polynomial in $M_L$ and in the branching parameters.

\section{Iterated versus flat wreath at $L = 2$}\label{app:iw_vs_flat}

The iterated construction at $L = 2$ with $(K_1, K_2) = (K, P)$ produces the same group $G_2 = S_K \wr S_P$ as the flat (C-W) construction. The two canonical sets differ in size:
\[
|\mathcal{C}^{\mathrm{IW}}_{K, P}| = (p(K) - 1) + (p(P) - 1),
\]
\[
|\mathcal{C}_{K,P}| \leq P(p(K) - 1) + (p(P) - 1) + \binom{P}{2}.
\]
The reduction is achieved through two structural observations.

\emph{(a) Within-block representatives in one block suffice.} The flat construction places canonical class representatives of $S_K$ in every block (the (T1) generators), giving $P(p(K) - 1)$ generators. The iterated construction places them only in the reference block (block $0$), giving $(p(K) - 1)$ generators. The remaining $K \times (P-1)$ block-$j$ within-block elements are recovered by conjugation: $\widetilde\pi_j \cdot c \cdot \widetilde\pi_j^{-1}$ maps the canonical representative $c$ acting in block $0$ to its analog in block $j$, where $\widetilde\pi_j \in S_{K_2}^{(\mathrm{block})} = S_P^{(\mathrm{block})}$ is generated by (T2)-style block permutations.

\emph{(b) Cross-block transpositions are redundant.} The flat construction includes (T3) generators $\widetilde\tau_{a, b}$ for every unordered block pair $\{a, b\}$, contributing $\binom{P}{2}$ generators. These are redundant once the (T2) generators are present: the canonical class representatives of $S_P$ generate $S_P$, and every transposition in $S_P$ is a product of canonical class representatives. The iterated construction omits (T3) entirely.

The proof of equivalence (both canonical sets generate $S_K \wr S_P$) follows by direct application of the coverage argument in Theorem~\ref{thm:iw_canonical} at $L = 2$ together with the flat-wreath coverage argument in Appendix~\ref{app:canonical_proofs}, (C-W).

\section{Closed-form Reynolds projectors}\label{app:reynolds}

For each named family $G$, the Reynolds projector $\mathcal{P}_G : \mathbb{C}^{M \times M} \to \mathcal{A}_G$ is defined by $\mathcal{P}_G(\mathbf{R}) = |G|^{-1} \sum_{g \in G} \rho(g) \mathbf{R} \rho(g)^*$, where $\rho$ is the permutation representation. The general formula has $|G|$ terms and can be exponential in $M$ to evaluate naively; the closed forms below collapse the sum to a polynomial number of orbit averages by exploiting the $G$-orbit structure on $\{1, \ldots, M\}^2$.

\emph{Cyclic.} For $G = \mathbb{Z}_n$, the orbits on $\{0, \ldots, n-1\}^2$ are indexed by the lag $k = (j - i) \bmod n$, giving $n$ orbits. The projector produces a circulant matrix with entries $c_k = n^{-1} \sum_{i=0}^{n-1} \mathbf{R}_{i, (i+k) \bmod n}$ for $k = 0, \ldots, n-1$. Total cost $O(n^2)$.

\emph{Dihedral.} For $G = D_n$, orbits are indexed by the folded lag $\min(k, n-k)$, giving $\lfloor n/2 \rfloor + 1$ orbits. The projector produces a symmetric circulant matrix with entries $c_{\min(k, n-k)} = (2n)^{-1} \sum_{i} (\mathbf{R}_{i, (i+k) \bmod n} + \mathbf{R}_{i, (i-k) \bmod n})$.

\emph{Direct product.} For $G = S_K^P$, orbits on pairs of points of $\{0, \ldots, K-1\} \times \{0, \ldots, P-1\}$ are indexed by ordered block pairs $(a, b)$ together with within-block diagonal versus off-diagonal type, giving $P(P+1)$ orbits. The projector preserves per-block heterogeneity by averaging within each orbit independently.

\emph{Flat wreath.} For $G = S_K \wr S_P$, orbits collapse further to three classes: within-block diagonal, within-block off-diagonal, and cross-block. The projector produces three orbit averages.

\emph{Iterated wreath.} For $G_L = (\cdots((S_{K_1} \wr S_{K_2}) \wr S_{K_3}) \cdots) \wr S_{K_L}$, the orbits on $\{0, \ldots, M_L - 1\}^2$ are indexed by the depth $d \in \{0, 1, \ldots, L\}$ of the deepest common ancestor of the index pair $(i, j)$ in the regular tree of branching $(K_1, \ldots, K_L)$. There are $L + 1$ orbits. The projector produces $L + 1$ orbit averages, generalizing the flat-wreath three-class structure ($L = 2$ gives the three classes within-block-diagonal, within-block-off-diagonal, cross-block, which correspond to depths $d = 0, 1, 2$). Total cost $O(M_L^2)$, polynomial.

\titleformat{\section}{\normalfont\large\scshape\centering}{\Roman{section}.}{0.5em}{}

\end{document}